\newif\ifextras
\extrastrue
\newif\iflong
\longtrue
\def\final{1}  

\newif\ifshort

\iflong
\shortfalse
\else
\shorttrue
\fi

\newif\ifcolt
\coltfalse
\ifcolt
\documentclass[letterpaper]{article}
  \usepackage{aaai17}
  \usepackage{times}
  \usepackage{helvet}
  \usepackage{courier}
  \usepackage{amsthm} 
\else
  \documentclass{article}
  \newcommand{\numberofauthors}[1]{} 
  \usepackage{amsthm} 
  \usepackage{hyperref}
  \usepackage{fullpage}
\fi

\iflong

\else

\fi

\usepackage{amsmath, amssymb, microtype}
\usepackage{bbm}
\usepackage{latexsym}
\usepackage{color}
\usepackage{caption}
\usepackage{subcaption}

\usepackage{silence}
\usepackage{algorithm}
\usepackage{algpseudocode}

\usepackage{float}

\floatstyle{ruled}
\newfloat{subroutine}{htbp}{loa}
\floatname{subroutine}{Subroutine}
\algdef{SE}[myFOR]{myFor}{myEndFor}[1]{\algorithmicfor\ #1\ \algorithmicdo}{\algorithmicend\ \algorithmicfor}%

\usepackage{mathtools}
\usepackage{amsfonts}
\usepackage{thmtools}
\usepackage{thm-restate}
\usepackage{verbatim}
\usepackage{multirow}
\usepackage{url}

\ifnum\final=0  
\newcommand{\lnote}[1]{[{\small Luis: \bf #1}]}
\newcommand{\jnote}[1]{[{\small Joe: \bf #1}]}
\newcommand{\nnote}[1]{[{\small Navin: \bf #1}]}
\newcommand{\anonnote}[1]{[{\small anon: \bf #1}]}
\newcommand{\sidecomment}[1]{\marginpar{\tiny #1}}
\newcommand{\details}[1]{{\color{blue}\ [[#1]] }}
\else 
\newcommand{\lnote}[1]{}
\newcommand{\jnote}[1]{}
\newcommand{\nnote}[1]{}
\newcommand{\anonnote}[1]{}
\newcommand{\sidecomment}[1]{}
\newcommand{\details}[1]{}
\fi  
\ifcolt
\title{Heavy-Tailed Analogues of the Covariance Matrix for ICA}
\author{ Joseph Anderson \\
The Ohio State University\\
andejose@cse.ohio-state.edu\\
\And
Navin Goyal \\
Microsoft Research\\
navingo@microsoft.com\\
\And
Anupama Nandi\\
The Ohio State University\\
nandi.10@osu.edu\\
\And
Luis Rademacher\\
University of California, Davis\\
lrademac@ucdavis.edu
}

\declaretheorem[name=Theorem]{theorem}
\newtheorem{lemma}[theorem]{Lemma}
\newtheorem{conjecture}[theorem]{Conjecture}
\newtheorem{proposition}[theorem]{Proposition}
\newtheorem{example}[theorem]{Example}
\newtheorem{corollary}[theorem]{Corollary}
\newtheorem{definition}[theorem]{Definition}
\newtheorem{fact}[theorem]{Fact}
\newtheorem{remark}[theorem]{Remark}

\newcommand{\mykeywords}[1]{\begin{keywords}#1\end{keywords}}

\else

\title{Heavy-Tailed Analogues of the Covariance Matrix for ICA}

\author{ Joseph Anderson\footnote{The Ohio State University, Department of Computer Science and Engineering. \texttt{andejose@cse.ohio-state.edu}} \and Navin Goyal\footnote{Microsoft Research, India. \texttt{navingo@microsoft.com}} \and Anupama Nandi\footnote{The Ohio State University, Department of Computer Science and Engineering. \texttt{nandi.10@osu.edu}}\and Luis Rademacher\footnote{University of California, Davis, Mathematics Department. \texttt{lrademac@ucdavis.edu}} }

\declaretheorem[name=Theorem]{theorem}
\newtheorem{lemma}[theorem]{Lemma}

\newtheorem{proposition}[theorem]{Proposition}

\newtheorem{definition}[theorem]{Definition}

\newcommand{\mykeywords}[1]{}
\fi

\newenvironment{proofidea}{\noindent{\textit{Proof idea.}}}{\hfill$\square$\medskip}

\renewcommand{\dim}{{n}}

\newcommand{\RR}{\ensuremath{\mathbb{R}}}
\newcommand{\QQ}{\ensuremath{\mathbb{Q}}}

\newcommand{\expectation}{\operatorname{\mathbb{E}}}
\newcommand{\e}{\expectation}
\newcommand{\conv}{\operatorname{conv}}
\newcommand{\suchthat}{\mathrel{:}}
\newcommand{\norm}[1]{{\lVert#1\rVert}}
\newcommand{\diag}{\operatorname{diag}}
\newcommand{\eps}{\epsilon}
\newcommand{\ind}{\mathbbm{1}}

\newcommand{\cov}{\operatorname{Cov}}

\newcommand{\poly}{\operatorname{poly}}

\newcommand{\abs}[1]{\lvert#1\rvert}

\newcommand{\noname}[1]{}

\newcommand{\polar}{\circ}

\newcommand{\measure}{\mathbb{P}}
\newcommand{\pr}{\mathbb{P}}
\newcommand{\inner}[2]{\langle{#1},{#2}\rangle}

\newcommand{\E}{\mathbb{E}}
\newcommand{\R}{\mathbb{R}}

\newcommand{\centroid}{\Gamma}

\newcommand{\ud}{\mathop{}\!\mathrm{d}}

\DeclareMathOperator{\sign}{sgn}

\iflong
\else
\renewenvironment{proof}{\expandafter\comment}{\expandafter\endcomment}
\fi

\usepackage[capitalise]{cleveref}

\allowdisplaybreaks

\begin{document}
\maketitle



\begin{abstract}
Independent Component Analysis (ICA) is the problem of learning a square matrix $A$%
, given samples of $X=AS$, 
where $S$ is a random vector with independent coordinates.
Most existing algorithms are provably efficient only when each $S_i$ has finite and moderately valued fourth moment. 
However, there are practical applications where this assumption need not be true, such as speech and finance. 
Algorithms have been proposed for heavy-tailed ICA, but they are not practical, using random walks and the full power of the ellipsoid algorithm multiple times.
The main contributions of this paper are: 

(1) A practical algorithm for heavy-tailed ICA that we call HTICA. 
We provide theoretical guarantees and show that it outperforms other algorithms in some heavy-tailed regimes, both on real and synthetic data.
Like the current state-of-the-art, the new algorithm is based on the centroid body (a first moment analogue of the covariance matrix). 
Unlike the state-of-the-art, our algorithm is practically efficient. To achieve this, we use explicit analytic representations of the centroid body, which bypasses the use of the ellipsoid method and random walks.

(2) We study how heavy tails affect different ICA algorithms, including HTICA. Somewhat surprisingly, we show that some algorithms that use the covariance matrix or higher moments can successfully solve a range of ICA instances with infinite second moment. We study this theoretically and experimentally, with both synthetic and real-world heavy-tailed data.
\end{abstract}
  
\frenchspacing

\section{Introduction}\label{sec:intro}

Independent component analysis (ICA) is a computational and statistical technique with applications in areas ranging from signal processing to machine learning and more. 
Formally, if $S$ is an $n$-dimensional random vector with independent coordinates and $A\in \RR^{n \times n}$ is invertible, then the ICA problem is to estimate $A$ given access to i.i.d. samples of the mixed signals $X = AS$. We say that $X$ is generated by 
an \emph{ICA model} $X=AS$. 
The recovery of $A$ (the \emph{mixing matrix}) is possible only up 
to scaling and permutation of the columns. Moreover, for the recovery to be possible, the distributions of the random
variables $S_i$ must not be Gaussian (except possibly one of them). 
Since its inception in the eighties (see 
\cite{ComonJutten} for historical remarks), ICA has been thoroughly 
studied and a vast literature exists (e.g. \cite{ICA01,ComonJutten}). 
The theory is well-developed and \emph{practical} algorithms---e.g., 
FastICA \cite{fastica99}, JADE \cite{CardosoS93}---are now available along with implementations,
e.g. \cite{ICALab}.
However, to our knowledge, rigorous complexity analyses of these assume that the fourth moment of each component is finite: $\E( S_i^4 )< \infty$.
If at least one of the independent components does not satisfy this assumption we will say that the input is in the \emph{heavy-tailed regime}.
Many ICA algorithms first preprocess the data to convert the given ICA model into another one where the mixing matrix $A$ has orthogonal columns; this step is often called \emph{whitening}. We will instead call it \emph{orthogonalization}, as this describes more precisely the desired outcome. Traditional whitening is a second order method that may not make sense in the heavy-tailed regime.
In this regime, it is not clear how the existing algorithms would 
perform, because they depend on empirical estimation of 
various statistics of the data such as the covariance matrix or the fourth cumulant 
tensor, which diverge in general for heavy-tailed data.
For example, for the covariance matrix in the mean-0 case this is done by taking the empirical average
$(1/N) \sum_{i=1}^{N} x(i) x(i)^T$
where the $\{x(i)\}$ are i.i.d. samples of $X$. 
ICA in the heavy-tailed regime is of considerable interest, directly 
(e.g., \cite{Kidmose01,KidmoseThesis,shereshevski2001super,ChenBickel04,chen2005consistent,sahmoudi2005blind,wang2009ica,eriksson2001novel,portfolio_extreme,bermond1999approximate}) and indirectly 
(e.g., \cite{bickson,gael2009infinite,welling2002learning}) and has 
applications in speech and finance. 
We also mention an informal connection with robust statistics: Algorithms
solving heavy-tailed ICA might work by focusing on samples in a small 
(but high probability) region to get reliable statistics about the data and avoid the instability of the tail. 
Thus, if the data has
outliers, the outliers are less likely to affect such an algorithm.

Recent theoretical work \cite{anon_htica} proposed a polynomial time algorithm for ICA that works in the regime where each component $S_i$ has finite $(1+\gamma)$-moment for $\gamma > 0$. 
This algorithm follows the two phases of several ICA algorithms: 
(i) Orthogonalize the independent components. 
The purpose of this step is to apply an affine transformation to the samples from $X$ so that the resulting samples correspond to an ICA model where the unknown matrix $A$ has orthogonal columns.
(ii) Learn the matrix with orthogonal columns. 
Each of these two phases required new techniques: 
(1) \emph{Orthogonalization via uniform distribution in the centroid body.} 
The input is assumed to be samples from an ICA model $X = AS$ where each $S_i$ is symmetrically distributed (w.l.o.g, see Sec.~\ref{sec:prelims}) and has at least $(1+\gamma)$-moments. 
The goal is to construct an \emph{orthogonalization} matrix $B$ so that $BA$ has orthogonal columns. 
In \cite{anon_htica}, the inverse of the square root of the covariance matrix of the uniform distribution in the centroid body is one such matrix. 
(2) \emph{Gaussian damping.} 
The previous step allows one to assume that the mixing matrix $A$ is orthogonal.
The modified second step is: 
If $X$ has density $\rho_X(t)$ for
$t \in \R^n$, then the algorithm constructs another ICA model $X_R=AS_R$ where $X_R$ has pdf proportional to
$\rho_X(t) \exp({-\norm{t}_2^2/R^2})$, where $R>0$ is a parameter chosen by the algorithm. This explains the term 
Gaussian damping.
This achieves two goals: 
(1) All moments of $X_R$ and $S_R$ are finite. 
(2) The product structure of is retained.
This follows from two facts: $A$ has orthogonal columns, and the Gaussian has independent components in any orthonormal basis.
Because of these properties, the model can be solved by traditional ICA algorithms. 
 
The algorithm in \cite{anon_htica} is theoretically efficient but impractical.
Their orthogonalization 
uses the ellipsoid algorithm for linear programming, which is not practical. It is not clear how to replace their use of the ellipsoid algorithm by practical linear programming tools, as their algorithm only has oracle access to a sort of dual and not an explicit linear program. 
Moreover, their orthogonalization technique uses samples uniformly distributed in the centroid body, generated by a random walk. 
This is computationally efficient in theory but, to the best of our knowledge, only efficient in practice for moderately low dimension.

\textbf{Our contributions.} 
Our contributions are experimental and theoretical. We provide a new and practical ICA algorithm, HTICA, building upon the previous theoretical work in \cite{anon_htica}.
HTICA works as follows:
  (1) Compute an orthogonalization matrix $B$. 
  (2) Pre-multiply samples by $B$ to get an orthogonal model.
  (3) Damp the data, run an existing ICA algorithm.
For step (1), we propose two theoretically sound and practically efficient ways below, \emph{orthogonalization via centroid body scaling} and \emph{orthogonalization via covariance}.
Our algorithm is simpler and more efficient, but needs a more technical justification than the method in \cite{anon_htica}.
We demonstrate the effectiveness of HTICA on both synthetic and real-world data. 

\emph{Orthogonalization via centroid body scaling.} 
We propose a more practical orthogonalization matrix than the one from  \cite{anon_htica} (orthogonalization via the uniform distribution in the centroid body, mentioned before).
First, consider the centroid body of random vector $X$, denoted $\Gamma X$ (this is really a function of the \emph{distribution} of $X$; formal definition in 
Sec.~\ref{sec:prelims}). For intuition, it is helpful to think of the centroid body as an ellipsoid whose axes are aligned with the independent components of $X$.
  The centroid body is in general not an ellipsoid, but it has certain symmetries aligned with the independent components. 
Let random vector $Y$ be a scaling of $X$ along every ray so that points at infinity are mapped to the boundary of $\Gamma X$, the origin is mapped to itself and the scaling interpolates smoothly. One such scaling is obtained in the following way: 
It is helpful to consider how far a point is in its ray with respect to the boundary of $\Gamma X$. This is given by the \emph{Minkoswki functional} of $\Gamma X$, denoted $p:\RR^n \to \RR$, which maps the boundary of $\Gamma X$ to 1 and interpolates linearly along every ray.
We can then achieve the desired scaling by first mapping a given point to the boundary point on its ray (the mapping $x \mapsto x/p(x)$) and then using the function $\tanh$, which maps $[0,\infty)$ to $[0,1]$ with $\tanh(0)=0$ and $\lim_{x \to \infty} \tanh(x) = 1$ to determine the final scale along the ray, namely, $\tanh p (x)$.
More formally, our scaling is the following: 
Let $Y$ be $\frac{\tanh p(X)}{p(X)} X $.
We show in Sec.~\ref{sec:centroidbody} that $B = \cov(Y)^{-1/2}$ is an orthogonalization matrix when $\cov(Y)$ is invertible.
In order to make this practical, one needs a practical estimator of the Minkowski functional of $\Gamma X$ from a sample of $X$.
In Sec.~\ref{sec:centroidbody} and \ref{sec:direct_membership}, we describe such an  algorithm and provide a theoretical justification, including finite sample estimates.
The proposed algorithm is much simpler and practical than the one described in \cite{anon_htica}. 
In particular, it avoids the use of the ellipsoid algorithm by the use of a closed-form linear programming representation of the centroid body (Prop.~\ref{prop:dualcharacterization}, Lemma \ref{lemma:centroid-lp}) and new approximation guarantees between the empirical (sample estimate) and true centroid body of a heavy-tailed distribution.
In Sec.~\ref{sec:centroidbody}, we discuss our practical implementation and show results where orthogonalization via centroid body scaling produces results with smaller error.

\emph{Orthogonalization via covariance.}
Previously, (e.g., in \cite{ChenBickel04}), the empirical covariance matrix was used for whitening in the 
heavy-tailed regime and, surprisingly, worked well in some situations. Unfortunately, the understanding of this was 
quite limited
. We give a theoretical explanation for this phenomenon in a fairly general heavy-tailed regime:
Covariance-based orthogonalization works well when each component $S_i$ has finite $(1+\gamma)$-moment, where $\gamma > 0$. 
We also study this algorithm in experimental settings. As we will see, while 
orthogonalization via covariance improves over previous algorithms, in general 
orthogonalization 
via centroid body has better performance because it has better numerical stability; but there are some situations
where orthogonalization via covariance matrix is better. 

\emph{Empirical Study.}
We perform experiments on both synthetic and real data to see the effect of heavy-tails on ICA.

In the synthetic data setting, we generate samples from a fixed heavy-tailed distribution
and study how well the algorithm can recover a random mixing matrix
(Sec.~\ref{sec:experiments}).
To study the algorithm with real data, we use recordings of human speech
provided by \cite{uky_data}.
This involves a room with different arrangements of microphones, and six humans speaking independently.
The speakers are recorded individually, so we can artificially mix them and have access to a ground truth.
We study the statistical properties of the data, observing that it does indeed behave as if the underlying processes are heavy-tailed.
The performance of our algorithm shows improvement over using FastICA on its own.

\section{Preliminaries}\label{sec:prelims}

Heavy-tailed distributions arise in a wide range of applications (e.g., \cite{nolan:2015}). 
They are characterized by the slow decay of their tails. 
Examples of heavy-tailed distributions include the Pareto and log-normal distributions. 

We denote the pdf of random variable $Z$ by $\rho_Z$.
We will assume that our distributions are symmetric, that is $\rho(x)=\rho(-x)$ for $x \in \R$. 
As observed in \cite{anon_htica}, this is without
loss of generality for our purposes. This follows from the fact that if $X=AS$ is an ICA model, and if we let
$X'=AS'$ be an i.i.d. copy of the same model, then $X-X' = A (S-S')$ is an ICA model with components of
$S-S'$ having symmetric pdfs. One further needs to check that if the components of $S$ are away 
from Gaussians then the same holds for $S-S'$; see \cite{anon_htica}.
We formulate our
algorithms for the symmetric case; the general case immediately reduces to the symmetric case.



For $K \subseteq \RR^n$, $K_\eps$ denotes the set of points that are at distance at most $\eps$ from $K$.
The set $K_{-\eps}$ is all points for which an $\eps$-ball around them is still contained in $K$.
\iflong 
The $\dim$-dimensional $\ell_p$ ball is denoted as $B_p^\dim$.
\fi

An important related family of distributions
is that of stable distributions (e.g., \cite{nolan:2015}). 
In general, the density of a stable distribution has no closed form, but is fully defined by four real-valued parameters.
Some stable distributions do admit a closed form, such as the Cauchy and Gaussian distributions.
For us the most important parameter is $\alpha \in (0,2]$, known as the stability parameter; we will think of the other three parameters as being fixed to constants. 

We use the notation $\poly(\cdot)$ to indicate a function which is asymptotically upper bounded by a polynomial expression of the given variables.

\iflong 
If $\alpha=2$, the distribution is Gaussian (the only non-heavy-tailed stable distribution), 
and if $\alpha=1$, it is the Cauchy distribution. 

\fi

\begin{definition}[Centroid body
]\label{def:centroid-body2}
Let $X \in \RR^\dim$ be a random vector with finite first moment, that is, for all $u \in \RR^\dim$ we have $\e (\abs{\inner{u}{X}}) < \infty$. 
The \emph{centroid body} of $X$ is the compact convex set, denoted $\Gamma X$, whose support function is $h_{\Gamma X}(u) = \e (\abs{\inner{u}{X}})$.
For a probability measure $\measure$, we define $\Gamma \measure$, the centroid body of $\measure$, as the centroid body of any random vector distributed according to $\measure$.
\end{definition}

Note that for the centroid body to be well-defined, the mean of the data must be finite.
This excludes, for instance, the Cauchy distribution from consideration in the present work.

\section{HTICA and experiments}\label{sec:experiments}

In this section, we show experimentally that heavy-tailed data poses a significant challenge for current ICA algorithms, and compare them with HTICA in different settings.
We observe some clear situations where heavy-tails seriously affect the standard ICA algorithms, and that these problems are frequently avoided by using the heavy-tailed ICA framework.
In some cases, HTICA does not help much, but maintains the same performance of plain FastICA.

To generate the synthetic data, we create a simple heavy-tailed density function
$f_{\eta}(x)$ proportional to ${(\abs{x}+1.5)^{-\eta}}$,
which is symmetric, and for $\eta > 1$, $f_\eta$ is the density of a distribution which has finite $k<\eta-1$ moment.
The signal $S$ is generated with each $S_i$ independently distributed from $f_{\eta_i}$.
The mixing matrix 
$A \in \RR^{\dim \times \dim}$ is generated with each coordinate i.i.d. $\mathcal{N}(0,1)$, columns normalized to unit length.
To compare the quality of recovery, the columns of the estimated mixing matrix, $\tilde{A}$ are permuted to align with the closest matching column of $A$, via the Hungarian algorithm.
We use the Frobenius norm to measure the error, but all experiments were also performed using the well-known Amari index \cite{amari1996new}; the results have similar behavior and are not presented here.


\subsection{Heavy-tailed ICA when $A$ is orthogonal: Gaussian damping and experiments}
\label{sec:ica-orthogonal}

Focusing on the third step above, where the mixing matrix already has orthogonal columns, ICA algorithms already suffer dramatically from the presence of heavy-tailed data.
As proposed in \cite{anon_htica}, Gaussian damping is a preprocessing technique that converts data from an ICA model $X=AS$, where 
$A$ is unitary (columns are orthogonal with unit $l_2$-norm)
to data from a related ICA model $X_R = A S_R$, where $R > 0$ is a parameter to be chosen.
The independent components of $S_R$ have finite moments of all orders and so the existing algorithms can 
estimate $A$.


Using samples of $X$, we construct the damped random variable $X_R$, with pdf $\rho_{X_R}(x) \propto \rho_{X}(x) \exp({-\norm{x}^2/R^2})$.
To normalize the right hand side, we can estimate
\[K_{X_R} = \e \exp({-\norm{X}^2/R^2})\]
so that
\[ \rho_{X_R}(x) =  \rho_{X}(x) \exp({-\norm{x}^2/R^2})/K_{X_R}. \]
If $x$ is a realization of $X_{R}$, then $s = A^{-1}x$ is a realization of the random variable $S_R$ and we have that $S_R$ has pdf $\rho_{S_R}(s) = \rho_{X_R}(x)$.
To generate samples from this distribution, we use rejection sampling on samples from $\rho_{X}$. 
When performing the damping, we binary search over $R$ so that about 25\% of the samples are rejected.
For more details about the technical requirements for choosing $R$, see \cite{anon_htica}.

\begin{figure*}[t]
  \centering
    \includegraphics[width=0.35\textwidth]{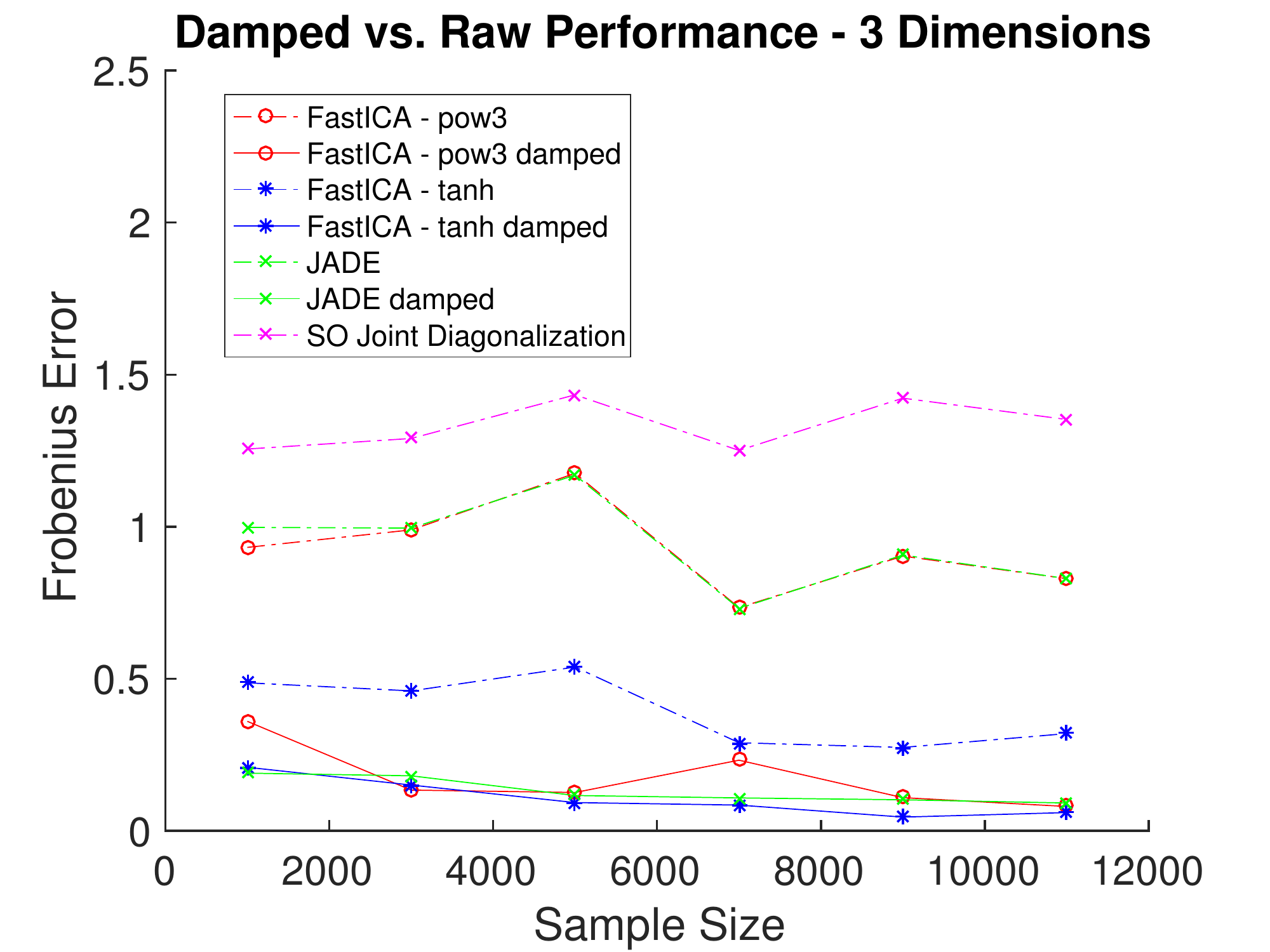}
    \includegraphics[width=0.35\textwidth]{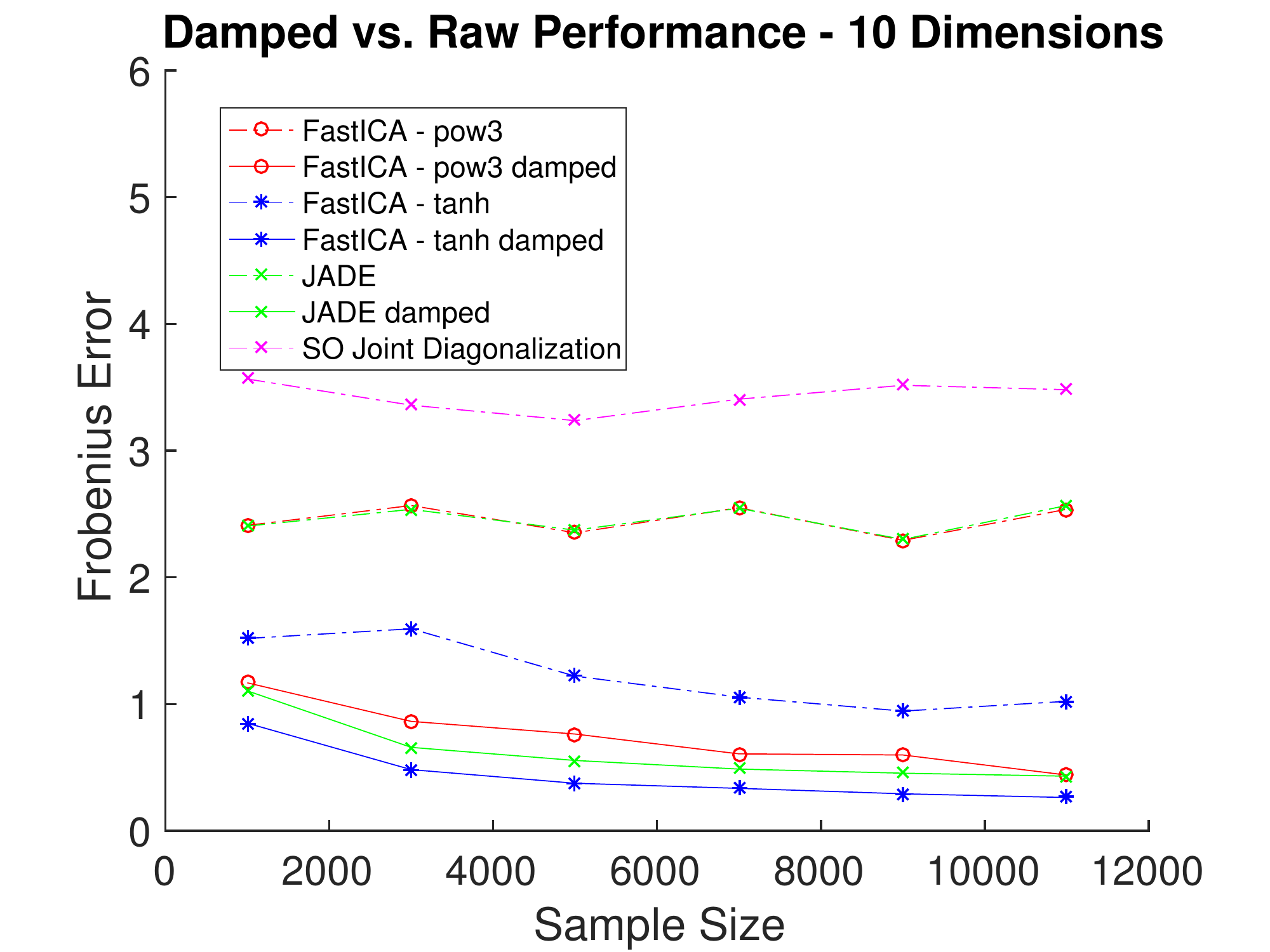}
    \includegraphics[width=0.26\textwidth]{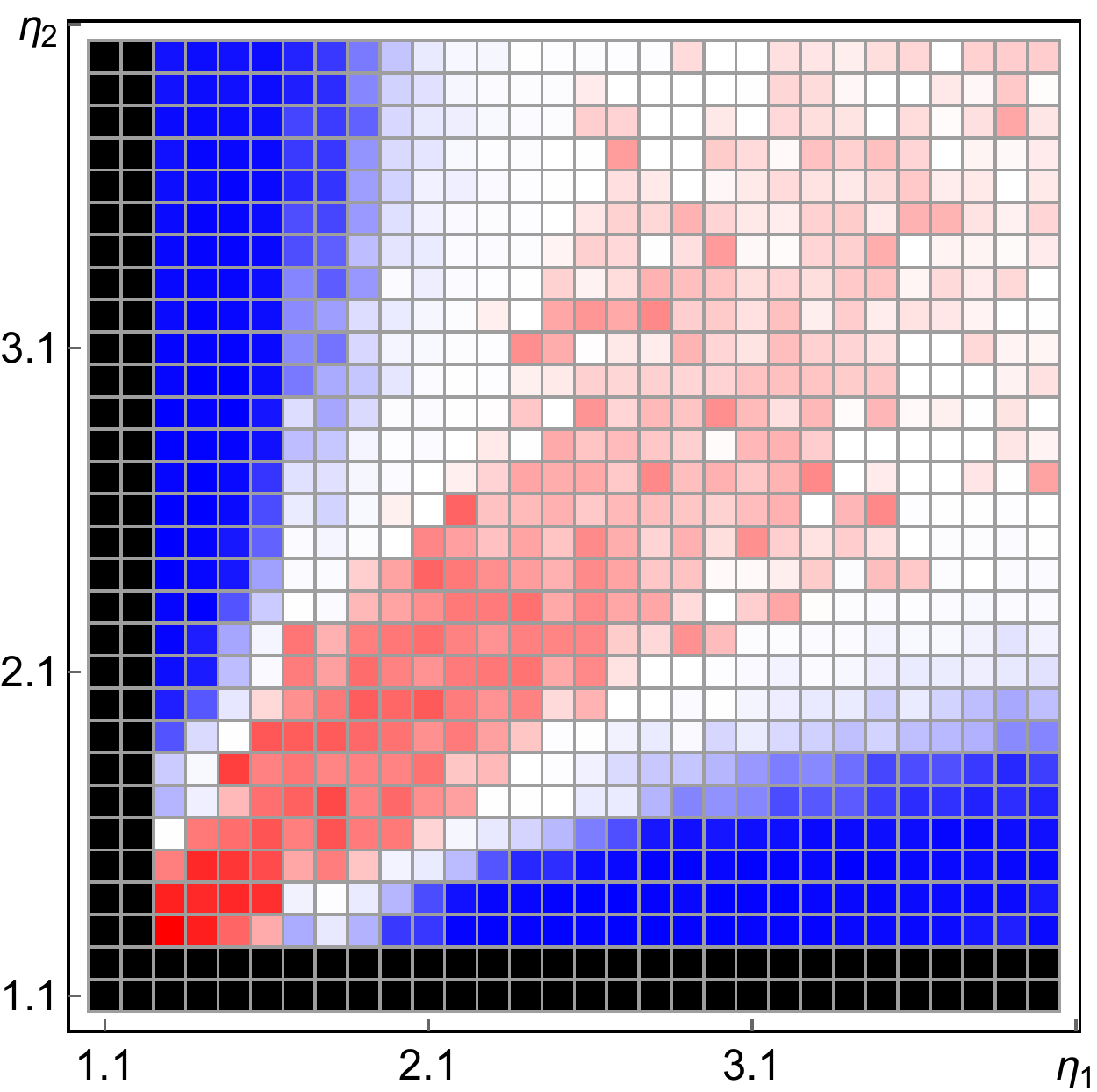}
    \caption{(Left, middle): The error of ICA with and without damping (solid lines and dashed lines, resp.), with unitary mixing matrix. The error is averaged over ten trials, in 3 and 10 dimensions where $\eta = (6,6,2.1)$ and $\eta = (6,\dotsc,6,2.1,2.1)$, resp.
    (Right): The difference between the errors of FastICA with and without damping in 2 dimensions, averaged over 40 trials. 
    For a single cell, the parameters are given by the coordinates, $\eta = (i,j)$. Red indicates that FastICA without damping does better than FastICA with damping, white indicates that the error difference is 0 and the blue indicates that FastICA with damping performs better than without damping. Black indicates that FastICA without damping failed (did not return two independent components).
    }\label{fig:orthogonal-mixedexp}
\end{figure*}

Figure~\ref{fig:orthogonal-mixedexp} shows that, when $A$ is already a perfectly orthogonal matrix, but where $S$ may have heavy-tailed coordinates, several standard ICA algorithms perform better after damping the data. In fact, without damping, some do not appear to converge to a correct solution.
We compare ICA with and without damping in this case: (1) FastICA using the fourth cumulant (``FastICA - pow3''), (2) FastICA using $\log \cosh$ (``FastICA - tanh''), (3) JADE, and (4) Second Order Joint Diagonalization as in, e.g., \cite{cardoso89}
.

\subsection{Experiments on synthetic heavy-tailed data}\label{sec:synthetic-data}\label{sec:synthetic-data-experiments}

We now present the results of HTICA using different orthogonalization techniques: (1)  Orthogonalization via \emph{covariance} (Section \ref{sec:covariance} (2) Orthogonalization via the \emph{centroid} body (Section~\ref{sec:centroidbody}) (3) the ground truth, directly inverting the mixing matrix (\emph{oracle}), and (4) No orthogonalization, and also no damping (for comparison with plain FastICA) (\emph{identity}).

\begin{figure*}[t]
  \centering
    \includegraphics[width=0.32\textwidth]{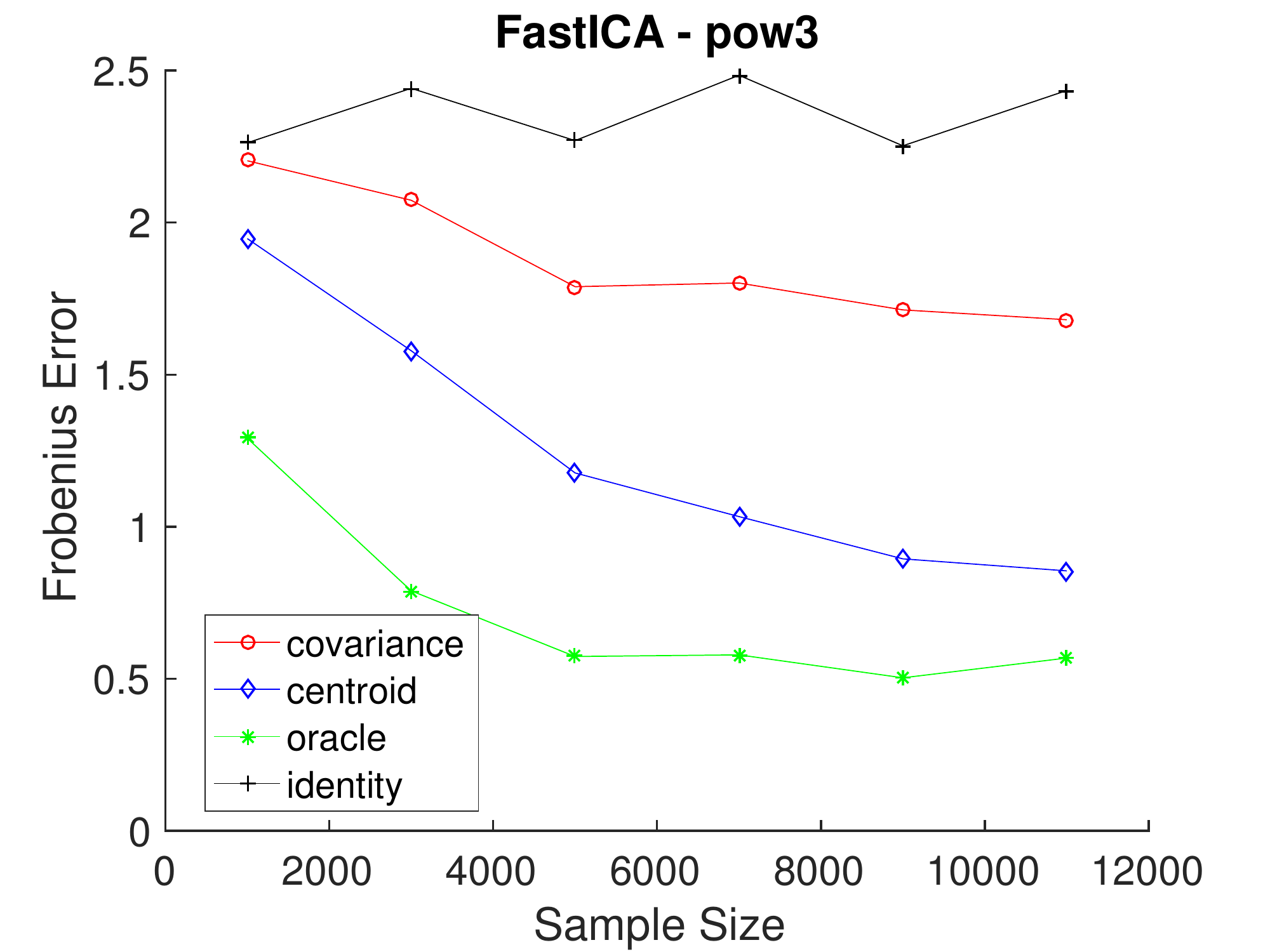}
    \includegraphics[width=0.32\textwidth]{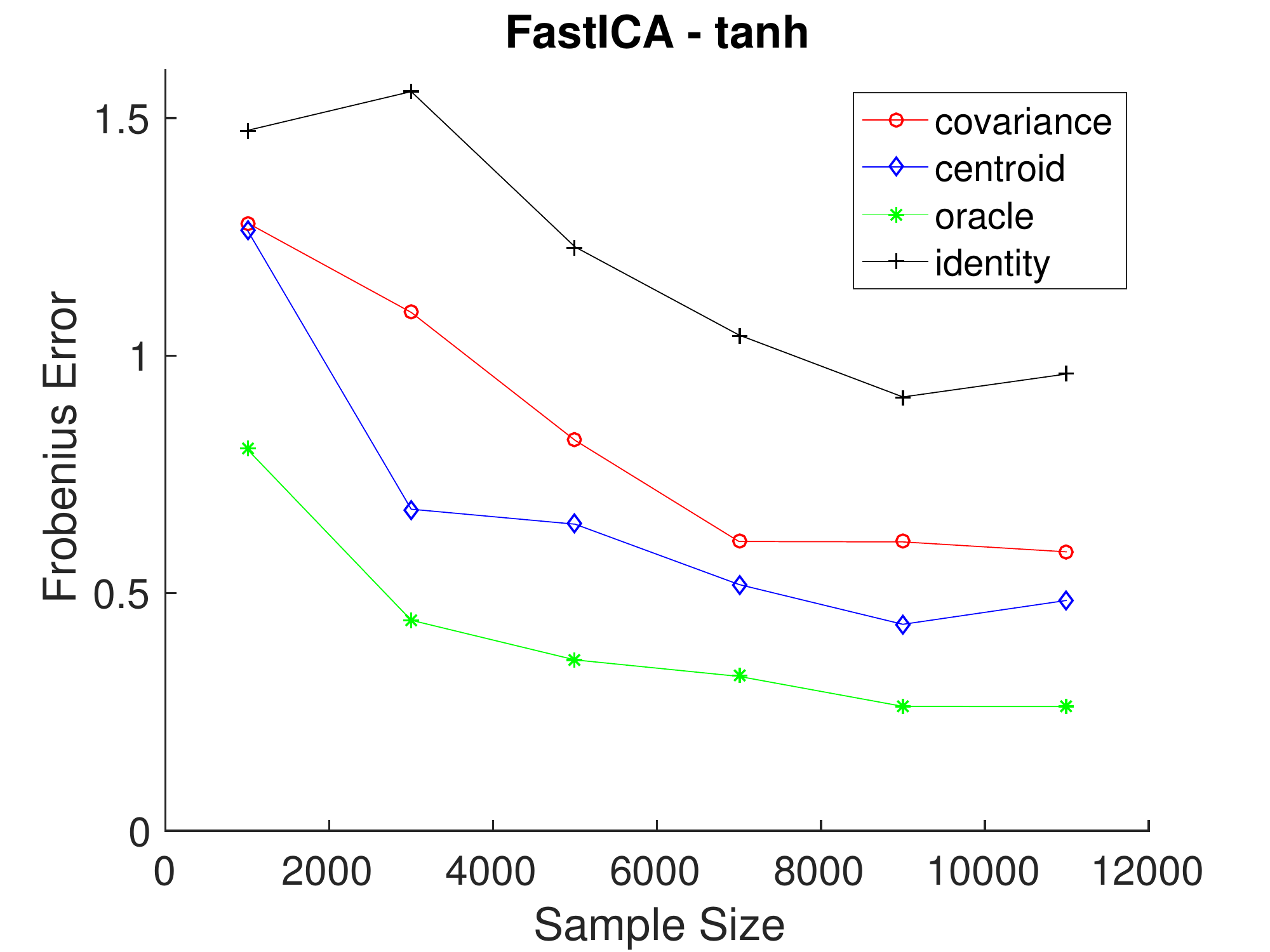}
    \includegraphics[width=0.32\textwidth]{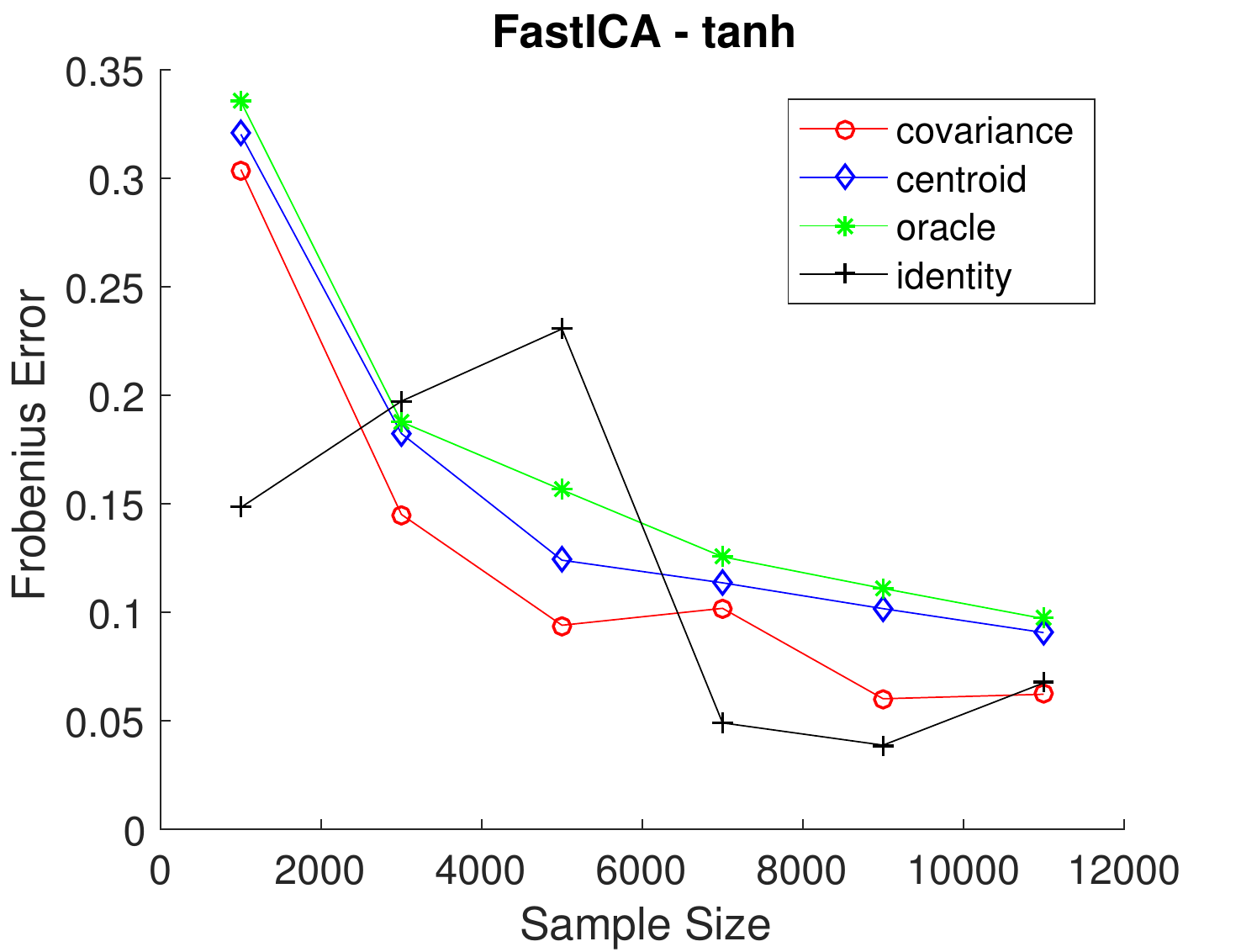}
    \caption{The Frobenius error of the recovered mixing matrix with the `pow3' and `tanh' contrast functions, on 10-dimensional data, averaged over ten trials. The mixing matrix $A$ is random with unit norm columns, not orthogonal.
    In the left and middle figures, the distribution has $\eta = (6,\dotsc,6,2.1,2.1)$ while in the right figure, $\eta = (2.1,\dotsc,2.1)$ (see Section \ref{sec:synthetic-data-experiments} for a discussion).
    }\label{fig:nonorthogonal-mixedexp}
\end{figure*}

The ``mixed'' regime in the left and middle of Figure~\ref{fig:nonorthogonal-mixedexp} (where some signals are \emph{not} heavy-tailed)
demonstrates a very dramatic contrast between different orthogonalization methods, even when only two heavy-tailed signals are present.

In the experiment with different methods of orthogonalization it was observed that when all exponents are the same or very close, orthogonalization via covariance performs better than orthogonalization via centroid and the true mixing matrix as seen in Figure~\ref{fig:nonorthogonal-mixedexp}.
A partial explanation is that, given the results in Figure~\ref{fig:orthogonal-mixedexp}, we know that equal exponents favor FastICA without damping and orthogonalization (\emph{identity} in Figure \ref{fig:nonorthogonal-mixedexp}). 
The line showing the performance with no orthogonalization and no damping (``identity'') behaves somewhat erratically, most likely due the presence of the heavy-tailed samples.
Additionally, damping and the choice of parameter $R$ is sensitive to scaling. A scaled-up distribution will be somewhat hurt because fewer samples will survive damping.

\subsection{ICA on speech data}\label{sec:voice-data}

While the above study on synthetic data provides interesting situations where heavy-tails can cause problems for ICA, we provide some results here which use real-world data, specifically human speech.
To study the performance of HTICA on voice data, we first examine whether the data 
is heavy-tailed.
The motivation to use speech data comes from observations by the signal processing community (e.g. \cite{kidmose2000alpha}) that speech data can be modeled by $\alpha$-stable distributions.
For an $\alpha$-stable distribution, with $\alpha \in (0,2)$, only the moments of order less than $\alpha$ will be finite.
We present here some results on a data set of human speech according to the standard cocktail party model, from \cite{uky_data}.

\iflong
The physical setup of the experiments (the human speakers and microphones) is shown in Figure~\ref{fig:voice-data-layouts}.

\begin{figure}
  \centering
  \includegraphics[width=0.35\textwidth]{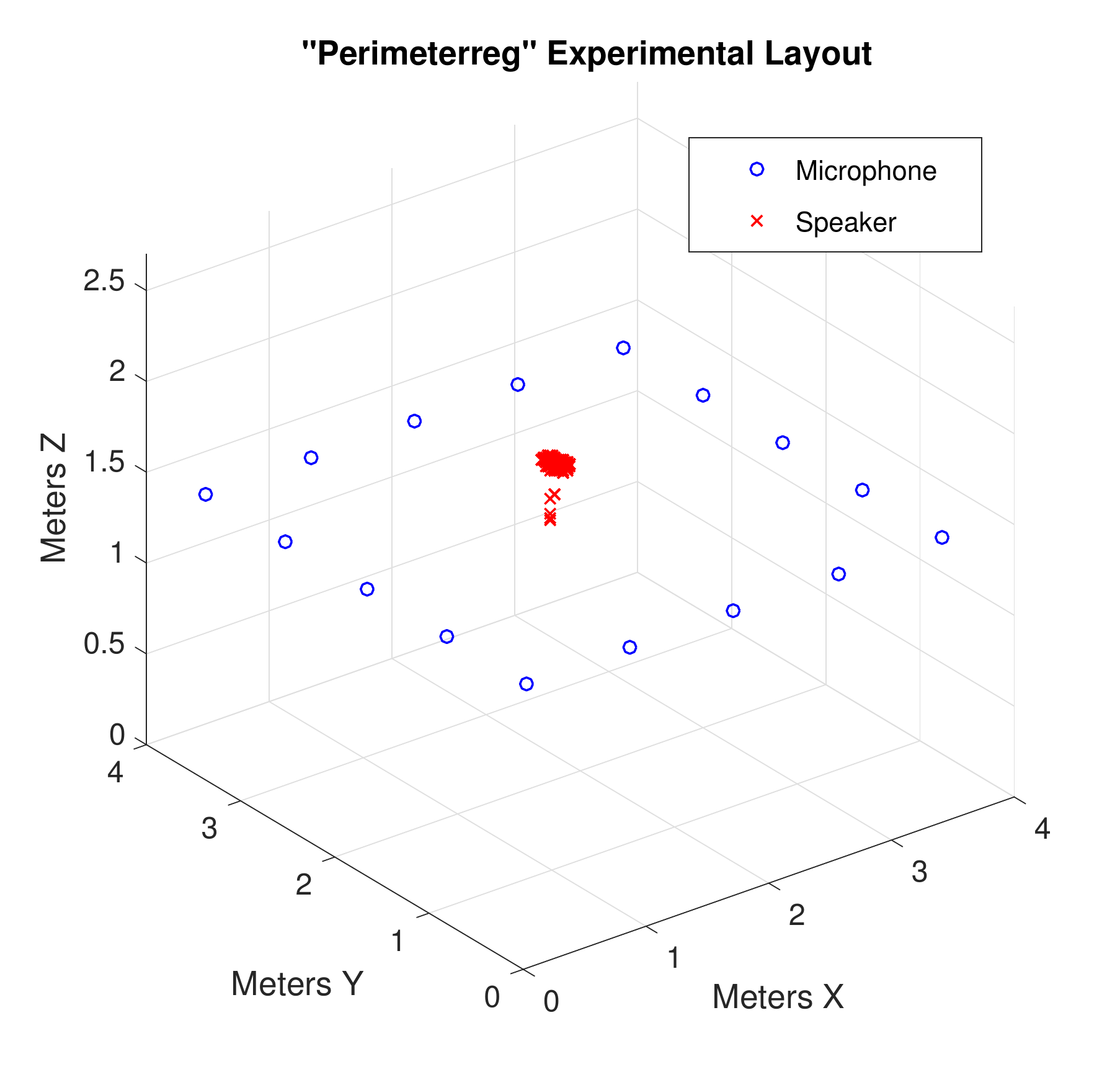}
  \includegraphics[width=0.35\textwidth]{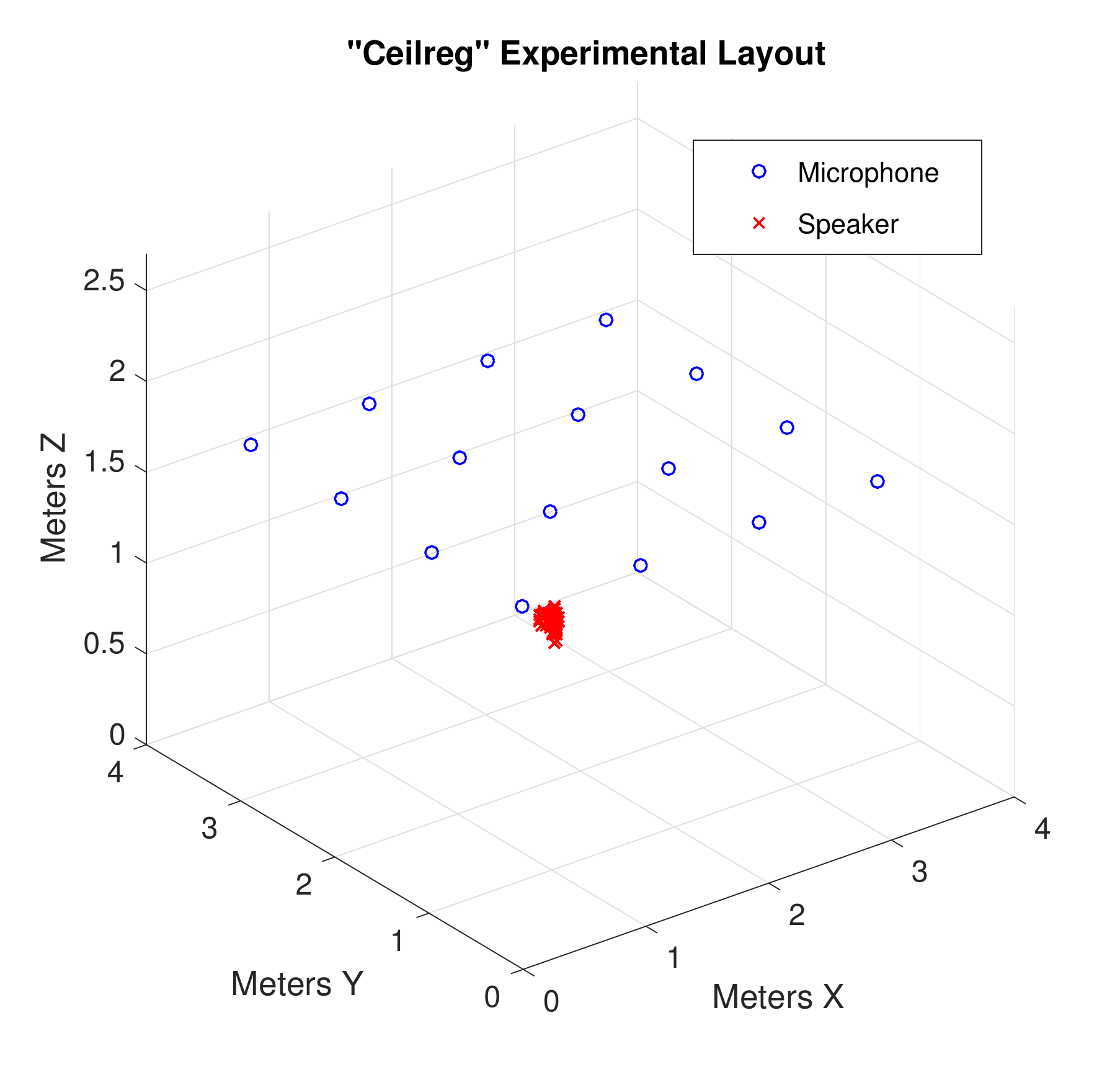}
  \caption{Microphone (blue circles) and human speaker (red ``x'') layouts for the ``ceilreg'' and ``perimeterreg'' voice separation data sets.}
  \label{fig:voice-data-layouts}
\end{figure}
\fi

To estimate whether the data is heavy-tailed, as in \cite{kidmose2000alpha}, we estimate parameter $\alpha$ of a best-fit $\alpha$-stable distribution.
This estimate is in Figure~\ref{fig:stable-parameter-estimation-and-error} for one of the data sets collected.
We can see that the estimated $\alpha$ is  clearly in the heavy-tailed regime for some signals.
\jnote{could use more references here.}
\jnote{include in an appendix the actual estimator formula.}

\begin{figure*}[t]
  \begin{minipage}{0.25\textwidth}
  \centering
  \includegraphics[width=0.99\textwidth]{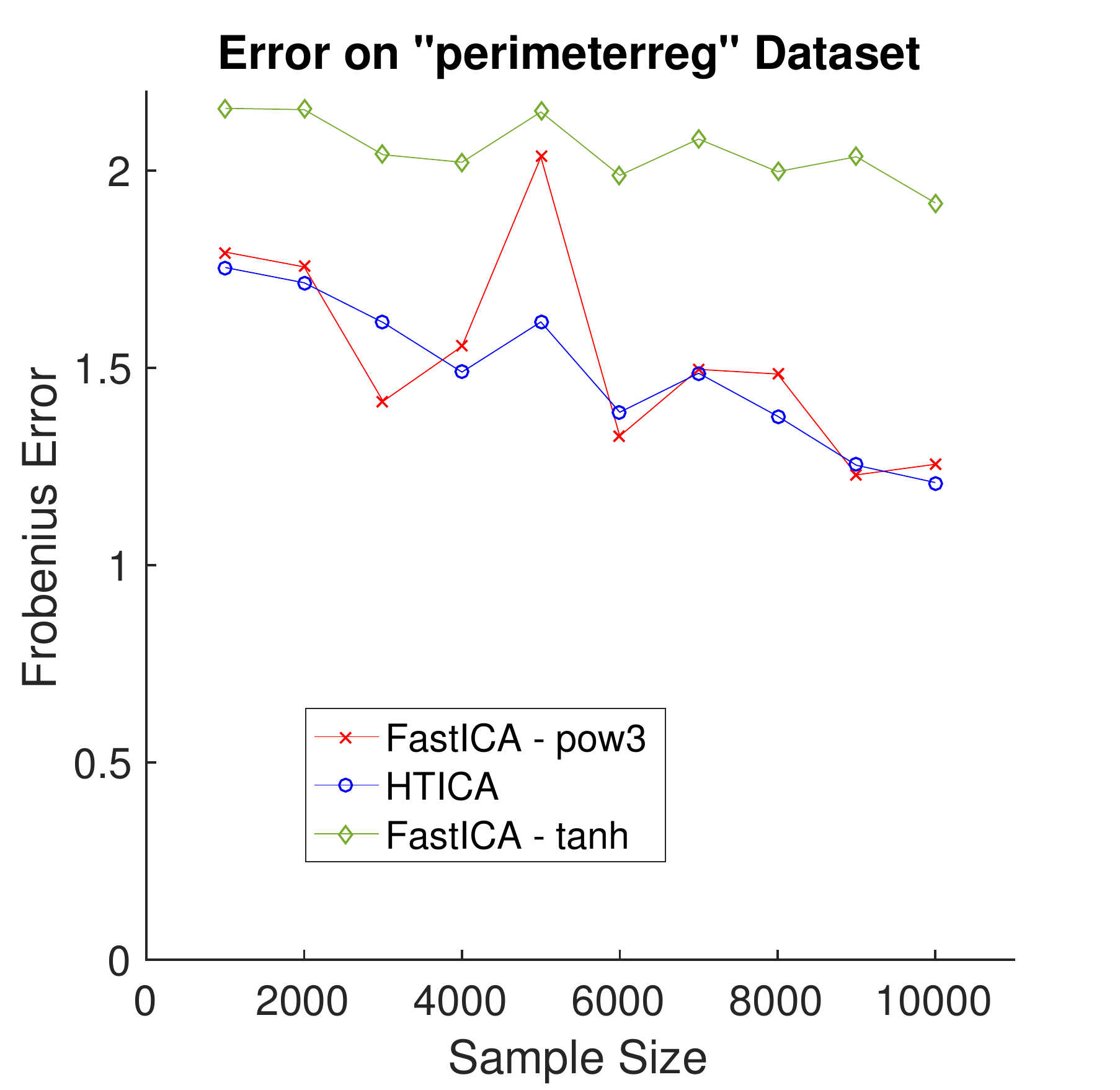}
  \end{minipage}
  \begin{minipage}{0.5\textwidth}
    \hspace{0.5cm}
    \resizebox{0.24\textwidth}{!}{%
      \begin{tabular}{|l|l|}
      \hline
      Signal & $\hat{\alpha}$ \\ \hline
      1 & 1.91 \\ \hline
      2 & 1.93 \\ \hline
      3 & 1.61 \\ \hline
      4 & 1.60 \\ \hline
      5 & 1.92 \\ \hline
      6 & 2.00 \\ \hline
      \end{tabular}
    }
    \hspace{0.7cm}
    \resizebox{0.9\textwidth}{!}{%
 \begin{tabular}{|c|l|l|l|l|l|}
\hline
\multicolumn{2}{|c|}{}               & \multicolumn{2}{l|}{$\sigma_{\min}$} & \multicolumn{2}{l|}{$\sigma_{\max}/\sigma_{\min}$} \\ \hline
\multicolumn{2}{|l|}{Orthogonalizer} & Centroid         & Covariance        & Centroid                & Covariance               \\ \hline
\multirow{6}{*}{Samples} & 1000  & 0.9302           & 0.9263            & 27.95                   & 579.95                   \\ \cline{2-6} 
                             & 3000  & 0.9603           & 0.9567            & 20.44                   & 410.11                   \\ \cline{2-6} 
                             & 5000  & 0.9694           & 0.9673            & 19.25                   & 490.11                   \\ \cline{2-6} 
                             & 7000  & 0.9739           & 0.9715            & 18.90                   & 347.68                   \\ \cline{2-6} 
                             & 9000  & 0.9790           & 0.9708            & 20.12                   & 573.18                   \\ \cline{2-6} 
                             & 11000 & 0.9793           & 0.9771            & 18.27                   & 286.34                   \\ \hline
\end{tabular}
}
\end{minipage}

  \caption{(Left): Error of estimated mixing matrix on the ``perimeterreg'' data, averaged over ten trials. 
  HTICA is more robust than FastICA.
    (Middle):  Stability parameter $\alpha$ estimates of each component in the ``perimeterreg'' data. Values below 2 are in the heavy-tailed regime.
    (Right): Smallest singular value and condition number of the orthogonalization matrix $BA$ computed via the centroid body and the covariance.
  The data was sampled with parameter $\eta = (6,6,6,6,6,6,6,6,2.1,2.1)$.
  }
  \label{fig:stable-parameter-estimation-and-error}

\end{figure*}

Using data from \cite{uky_data}, we perform the same experiment as in Section~\ref{sec:synthetic-data-experiments}: generate a random mixing matrix with unit length columns, mix the data, and try to recover the mixing matrix.
Although the mixing is synthetic, the setting makes the resulting mixed signals same as real.
Specifically, the experiment was conducted in a room with chairs, carpet, plasterboard walls, and windows on one side.
There was natural noise including vents, computers, florescent lights, and traffic noise through the windows.

Figure~\ref{fig:stable-parameter-estimation-and-error} demonstrates that HTICA (orthogonalizing with centroid body scaling, Section~\ref{sec:centroidbody}) applied to speech data yields some noticeable improvement in the recovery of the mixing matrix, primarily in that it is less susceptible to data that causes FastICA to have large error ``spikes.''
Moreover, in many cases, running only FastICA on the mixed data failed to even recover all of the speech signals, while HTICA succeeded.
In these cases, we had to re-start FastICA until it recovered all the signals.


\section{New approach to orthogonalization and a new analysis of empirical covariance}\label{sec:orthogonalization}


As noted above, the technique in \cite{anon_htica}, while being provably efficient and correct, suffers from practical implementation issues.
Here we discuss two alternatives: orthogonalization by \emph{centroid body scaling} and orthogonalization by using the empirical covariance.
The former, orthogonalization via centroid body scaling, uses the samples already present in the algorithm rather than relying on a random walk to draw samples which are approximately uniform in the algorithm's approximation of the centroid body (as is done in \cite{anon_htica}).
This removes the dependence on random walks and the ellipsoid algorithm; instead, we use samples that are distributed according to the original heavy-tailed distribution but \emph{non-linearly scaled} to lie inside the centroid body.
We prove in Lemma \ref{lemma:cvar-orthogonalizer} that the covariance of this subset of samples is enough to orthogonalize the mixing matrix $A$.
Secondly, we prove that one can, in fact, ``forget'' that the data is heavy tailed and orthogonalize by using the empirical covariance of the data, even though it diverges, and that this is enough to orthogonalize the mixing matrix $A$. However, as observed in experimental results, 
in general this has a downside compared to orthogonalization via centroid body in that it could cause numerical instability during
the ``second'' phase of ICA as the data obtained is less well-conditioned. 
This is illustrated directly in the table in Figure~\ref{fig:stable-parameter-estimation-and-error} containing the singular value and 
condition number of the mixing matrix $BA$ in the approximately orthogonal ICA model.

\subsection{Orthogonalization via centroid body scaling}\label{sec:centroidbody}

In \cite{anon_htica}, another orthogonalization procedure, namely \emph{orthogonalization via the uniform distribution in the centroid body} is theoretically proven to work. 
Their procedure does not suffer from the numerical instabilities and composes well with the second phase of ICA algorithms. 
An impractical aspect of that procedure is that it needs samples from the uniform distribution in the centroid body.

We described orthogonalization via centroid body in Section \ref{sec:intro}, except for the estimation of $p(x)$, the Minkowski functional of the centroid body. The complete procedure is stated in Subroutine \ref{sub:centroid}. 

We now explain how to estimate the Minkowski functional. The Minkowski functional was informally described in Section \ref{sec:intro}. 
The Minkowski functional of $\Gamma X$ is formally defined by $p(x) := \inf \{t>0 \suchthat x \in t \Gamma X\}$.
Our estimation of $p(x)$ is based on an explicit linear program (LP) \eqref{eq:lpminkowski} that gives the Minkowski functional of the centroid body of a finite sample of $X$ \emph{exactly} and then arguing that a sample estimate is close to the actual value for $\Gamma X$. For clarity of exposition, we only analyze formally a special case of LP \eqref{eq:lpminkowski} that decides \emph{membership} in the centroid body of a finite sample of $X$ (LP \eqref{eq:lp}) and approximate membership in $\Gamma X$. 
This analysis is in Section \ref{sec:direct_membership}. 
Accuracy guarantees for the approximation of the Minkowski functional follow from this analysis.

\begin{subroutine}[ht]
\caption{Orthogonalization via centroid body scaling}
\label{sub:centroid}
\begin{algorithmic}[1]
\Require
Samples $(X^{(i)})_{i=1}^N$ of ICA model $X = AS$ so each $S_i$ is symmetric with $(1+\gamma)$ moments.
\Ensure 
Matrix $B$ approximate orthogonalizer of $A$ 
\For{$i = 1:N$}, 
\State Let $\lambda^*$ be the optimal value of \eqref{eq:lpminkowski} with $q = X^{(i)}$. Let $d_i = 1/\lambda^*$. Let $Y^{(i)} = \frac{\tanh d_i}{d_i} X^{(i)}$.
\EndFor
\State Let $C = \frac{1}{N} \sum_{i=1}^N Y^{(i)} {Y^{(i)}}^T$.
Output $B = C^{-1/2}$.
\end{algorithmic}
\end{subroutine}
\iflong
\begin{lemma}[\cite{anon_htica}]\label{lem:orthogonalizer}
Let $U$ be a family of $n$-dimensional product distributions. Let $\bar U$ be the closure of $U$ under invertible linear transformations.
Let $Q(\measure)$ be an $n$-dimensional distribution defined as a function of $\measure \in \bar U$. Assume that $U$ and $Q$ satisfy:
\begin{enumerate}
\item\label{item:sym}
For all $\measure \in U$, $Q(\measure)$ is absolutely symmetric.
\item\label{item:equivariant} $Q$ is linear equivariant (that is, for any invertible linear transformation $T$ we have $Q(T\measure) = T Q(\measure)$).
\item\label{item:positive}
For any $\measure \in \bar U$, $\cov(Q(\measure))$ is positive definite.
\end{enumerate}
Then for any symmetric ICA model $X=AS$ with $\measure_S \in U$ we have $\cov(Q(\measure_X))^{-1/2}$ is an orthogonalizer of $X$.
\end{lemma}
\fi

\begin{lemma}\label{lemma:cvar-orthogonalizer}
Let $X$ be a random vector drawn from an ICA model $X=AS$ such that for all $i$ we have $\e \abs{S_i} = 1$ and $S_i$ is symmetrically distributed. 
Let $Y = \frac{\tanh p(X)}{p(X)} X$ where $p(X)$ is the Minkoswki functional of $\Gamma X$.
\details{$\centroid X$ contains $A B_1^\dim$}
Then $\cov(Y)^{-1/2}$ is an orthogonalizer of $X$.
\end{lemma}

\begin{proof}
We will be applying Lemma \ref{lem:orthogonalizer}.
Let $U$ denote the set of absolutely symmetric product distributions $\measure_W$ over $\RR^\dim$ such that $\e \abs{W_i} = 1$ for all $i$.
For $\measure_V \in \bar U$, let $Q(\measure_V)$ be equal to the distribution obtained by scaling $V$ as described earlier, that is, distribution of $\alpha V$,  where $\alpha = \frac{\tanh p(V)}{p(V)} $, $p(V)$ is the Minkoswki functional of $\Gamma \measure_V $.

For all $\measure_W \in U$, $W_i$ is symmetric and $\e \abs{W_i} = 1$ which implies that $\alpha W$, that is, $Q(\measure_W)$ is absolutely symmetric.
Let $\measure_V \in \bar U$. Then $Q(\measure_V)$ is equal to the distribution of $\alpha V$. For any invertible linear transformation $T$ and measurable set  $\mathcal{M}$, we have 
$Q(T\measure_V)(\mathcal{M}) = Q(\measure_{TV})(\mathcal{M})
=\measure_{ \alpha TV}(\mathcal{M})
=\measure_{\alpha V}(T^{-1}\mathcal{M})
=TQ(\measure_{V})(\mathcal{M})$.
\details{$\alpha= \frac{\tanh p(TV)}{p(TV)}, p(TV) =$  Minkoswki functional of $\Gamma TV$, that is, $T \Gamma V$ from Lemma \ref{lem:equivariance} (linear equivariance of $\Gamma$). $p(T^{-1}TV) =$ Minkoswki functional of $T^{-1}T \Gamma V$}
Thus $Q$ is linear equivariant. 
Let $\measure \in \bar U$. Then there exist $A$ and $\measure_W \in U$ such that $\measure = A \measure_W$.  
We get $\cov(Q(\measure)) = \cov(AQ(\measure_W))$. Let $W_\alpha = \alpha W$.
Thus, $\cov(AQ(\measure_W)) = A \e (W_\alpha {W}_\alpha^T ) A^T$ where $\e (W_\alpha {W}_\alpha^T)$ is a diagonal matrix with elements $\e (\alpha^2 W_i^2 )$ which are non-zero because we assume $\e \abs{W_i} = 1$. 
This implies that $\cov(Q(\measure))$ is positive definite and thus by Lemma \ref{lem:orthogonalizer}, $\cov(Y)^{-1/2}$ is an orthogonalizer of $X$. 
\end{proof}

\subsection{Orthogonalization via covariance}\label{sec:covariance}

Here we show the somewhat surprising fact that  orthogonalization of heavy-tailed signals is sometimes possible by using the ``standard'' approach: inverting the empirical covariance matrix.
The advantage here, is that it is computationally very simple, specifically that having heavy-tailed data incurs very little computational penalty on the process of orthogonalization alone.
It's standard to use covariance matrix for \emph{whitening} when the second moments of all independent components 
exist \cite{ICA01}:
Given samples from the ICA model $X = AS$, we compute the empirical covariance matrix $\tilde\Sigma$ which 
tends to the true covariance matrix as we take more samples and set 
$B=\tilde\Sigma^{-1/2}$. Then one can show that $BA$ is a rotation matrix, and thus by pre-multiplying the data by
$B$ we obtain an ICA model $Y = BX = (BA)S$, where the mixing matrix $BA$ is a rotation matrix, and this model is then amenable to various algorithms. 
In the heavy-tailed regime where the second moment does not exist for some of the components, 
there is no true covariance matrix and the empirical covariance diverges as we take more samples. 
However, for any fixed number of samples one can still compute the empirical covariance matrix. 
In previous work (e.g., \cite{ChenBickel04}), the empirical covariance matrix was used for whitening in the heavy-tailed regime with good empirical performance; \cite{ChenBickel04} also provided some theoretical analysis to explain this surprising performance. 
However, their work (both experimental and theoretical) was limited to some very special cases 
(e.g., only one of the components is heavy-tailed, or there are only two components both with stable
distributions without finite second moment).

We will show that the above procedure 
(namely pre-multiplying the data by $B:=\tilde\Sigma^{-1/2}$)  
``works" under considerably more general conditions, namely
if $(1+\gamma)$-moment exists for $\gamma > 0$ for each independent component $S_i$. 
By ``works" we mean that instead of whitening the data (that is $BA$ is rotation matrix) it does something slightly weaker
but still just as good for the purpose of applying ICA algorithms in the next phase. It \emph{orthogonalizes} the 
data, that is now $BA$ is close to a matrix whose columns are orthogonal.
In other words, $(BA)^T(BA)$ is close to a diagonal matrix (in a sense made precise in 
Theorem~\ref{thm:alg-1-correctness}).

\iflong
Let $X$ be a real-valued symmetric random variable such that $\E (\abs{X}^{1+\gamma}) \leq M$ for some $M > 1$ and 
$0 < \gamma < 1$. The following lemma from \cite{anon_htica} says that the empirical average of the absolute value of $X$ converges to the expectation of $|X|$. The proof, which we omit, follows an argument similar to the proof of the Chebyshev's inequality. 
Let $\tilde{\E}_N[\abs{X}]$ be the empirical average obtained from $N$ independent samples 
$X^{(1)}, \ldots, X^{(N)}$,
i.e., $(\abs{X^{(1)}}+\dotsb+\abs{X^{(N)}})/N$.

\begin{lemma}\label{lem:1-plus-eps-chebyshev}
Let $\epsilon \in (0,1)$. With the notation above, for 
$N \geq \left(\frac{8M}{\epsilon}\right)^{\frac{1}{2}+\frac{1}{\gamma}}$,
we have
$\Pr[\abs{\tilde{\E}_N[\abs{X}]-\E[\abs{X}]} > \epsilon] \leq \frac{8M}{\epsilon^2 N^{\gamma/3}}$.
\end{lemma}
\fi
\begin{theorem}[Orthogonalization via covariance matrix]\label{thm:alg-1-correctness}
Let $X$ be given by ICA model $X=AS$. 
Assume that there exist $t, p, M > 0$  and $\gamma \in (0,1)$
such that for all $i$ we have

(a) $\e (\abs{S_i}^{1+\gamma}) \leq M < \infty$,

(b) (normalization) $\e \abs{S_i} = 1$, and

(c) $\Pr(\abs{S_i} \geq t) \geq p$.
Let $x^{(1)}, \dotsc, x^{(N)}$ be i.i.d. samples according to $X$. Let $\tilde{\Sigma} = (1/N) \sum_{k=1}^N x^{(k)} {x^{(k)}}^T$ and $B = \tilde{\Sigma}^{-1/2}$.
Then for any $\eps, \delta \in (0,1)$, $\norm{(BA)^T B A - D}_{2} \leq \eps$
for a diagonal matrix $D$ with diagonal entries $d_1, \dotsc, d_\dim$ satisfying $0 < d_i, 1/d_i \leq \max\{ 2/pt^2, N^4 \}$ for all $i$
with probability $1-\delta$ when
$N \geq \poly(n, M, 1/p, 1/t, 1/\eps, 1/\delta)$.
\end{theorem}
\iflong
\begin{proofidea}
For $i\neq j$ we have $\e(S_i S_j) = 0$ (due to our symmetry assumption on $S$) 
and $\e(\abs{S_i S_j}) = \e(\abs{S_i})\e(\abs{S_j}) < \infty$.
We have $(BA)^T BA = L^{-1}$, where $L = (1/N) \sum_{k=1}^N s^{(k)} {s^{(k)}}^T$. The off-diagonal entries of $L$ 
converge to $0$: We have $L_{i,j} = \E S_i S_j = (\E S_i)(\E S_j)$. Now by our assumption that $(1+\gamma)$-moments 
exist, Lemma~\ref{lem:1-plus-eps-chebyshev} is applicable and implies that empirical average $\tilde\E S_i$ tends
to the true average $\E S_i$ as we increase the number of samples. The true average is $0$ because of our assumption
of symmetry (alternatively, we could just assume that the $X_i$ and hence $S_i$ have been centered). 
The diagonal entries of $L$ are bounded away from $0$: This is clear when the second moment is finite,
and follows easily by hypothesis (c) when it is not. \nnote{is the last statement correct}
Finally, one shows that if in $L$ the diagonal entries highly dominate the off-diagonal entries, then the same
is true of $L^{-1}$.
\nnote{is that how the proof goes}
\end{proofidea}
\fi

\iflong

\begin{proof}
\lnote{fix indices}
We have $(BA)^T BA = L^{-1}$, where $L = (1/N) \sum_{k=1}^N s^{(k)} {s^{(k)}}^T$.
By assumption, $\e L_{ij} = 0$ for $i \neq j$.
Note that $\e \abs{s_i s_j}^{1+\gamma} \leq M^2$ and so by Lemma~\ref{lem:1-plus-eps-chebyshev}, for $i \neq j$,
\[
P(\abs{L_{ij}} > \eps_1) \leq \frac{8 M^2}{\eps_1^2 N^{\gamma/3}}
\]
when $N \geq (\frac{8 M^2}{\eps_1})^{\frac{1}{2} + \frac{1}{\gamma}}$.

Now let $D := \diag(L^{-1}_{11}, L^{-1}_{22}, \dots, L^{-1}_{\dim \dim})$.
Then when $\abs{L_{ij}} < \eps_1$ for all $i \neq j$, we have $\norm{L - D^{-1}}_{2} \leq \norm{L - D^{-1}}_{F} \leq n \eps_1$.
The union bound then implies
\begin{equation}\label{eq:off-diagonal-union-bound}
  \begin{aligned}
 P( \norm{L - D^{-1}}_{2} < n\eps_1 ) & \geq P( \norm{L - D^{-1}}_{F} < n\eps_1 ) \\
 & \geq P( \forall i\neq j, \abs{L_{ij}} \leq \eps_1) \\
 & \geq 1 - \frac{8 \dim^2 M^2}{\eps_1^2 N^{\gamma/3}}
  \end{aligned}
\end{equation}
when $N \geq (\frac{8 M^2}{\eps_1})^{\frac{1}{2} + \frac{1}{\gamma}}$.

Next, we aim to bound $\norm{D}_2$ which can be done by writing
\begin{equation}
\norm{D}_2 = \frac{1}{\sigma_{\min}(D^{-1})} = \frac{1}{\min_{i \in [n]} L_{ii}}
\end{equation}
where $L_{ii} = (1/N) \sum_{k=1}^{N} {s_i^{(k)}}^2$.
Consider the random variable $\ind(s_i^2 \geq t^2)$.
We can calculate $\e \sum_j \ind({s_i^{(j)}}^2 \geq t^2) \geq Np$ and use a Chernoff bound to see
\begin{equation}
P\left(\sum_{k \in [N]} \ind({s_i^{(k)}}^2 \geq t^2) \leq \frac{Np}{2} \right) \leq \exp\bigg(-\frac{Np}{8}\bigg)
\end{equation}
and when $\sum_{k \in [N]} \ind({s_i^{(k)}}^2 \geq t^2) \geq \frac{Np}{2}$, we have $L_{ii} \geq t^2 p/2$.
Then with probability at least $1-n\exp(-Np/8)$, all entries of $D^{-1}$ are at least $t^2p/2$.
Using this, if $N \geq N_1 := (8/p)\ln(3n/\delta)$ then $\norm{D}_{2} \leq 2/pt^2$ with probability at least $1 - \delta/3$.

Similarly, suppose that $\norm{D}_{2} \leq 2/pt^2$ and choose $\eps_1 = \min\{ \frac{t^4 p^2}{4n} \cdot \frac{\eps}{2}, \frac{1}{pt^2} \}$ and
\[
N_2 := \max\bigg\{ \bigg( \frac{24n^2 M^2}{\eps_1^2 \delta} \bigg)^{3/\gamma}, \bigg(\frac{8 M^2}{\eps_1}\bigg)^{\frac{1}{2} + \frac{1}{\gamma}} \bigg\}
\]
so that when $N \geq N_2$, we have
$\norm{L - D^{-1}}_{2} \leq 1/(2 \norm{D}_{2})$ and $\norm{L - D^{-1}}_{2} \leq t^4 p^2 \eps / 8$ with probability at least $1-\delta/3$.
Invoking (\ref{eq:inverse-stability}), when $N \geq \max\{N_1, N_2\}$, we have
\begin{equation}\label{eq:inverse-bound}
\norm{L^{-1} - D}_{2} \leq 2 \norm{D}_{2} \norm{L-D^{-1}}_{2} \leq 2 \frac{4}{p^2 t^4} \frac{t^4 p^2 \eps}{8} = \eps
\end{equation}
with probability at least $1 - 2 \delta/3$.

Finally, we upper bound $1/d_i$ for a fixed $i$ by using Markov's inequality:
\begin{equation}
\begin{split}
P\left(\frac{1}{d_i} > N^5 \right) &= P(L_{ii} > N^4) = P\bigg( \sum_{j}^{N} {s_i^{(j)}}^2 > N^5 \bigg) \\
&\leq N P(S_i^2 > N^4 ) \leq N P(\abs{S_i} > N^2) \\ 
&\leq N \frac{\e \abs{S_i}}{N^2} = \frac{1}{N}
\end{split}
\end{equation}
so that $1/d_i \leq N^4$ for all $i$ with probability at least $1-\delta/3$ when $N \geq N_3 := n/3\delta$.
Therefore, when $N \geq \max\{N_1, N_2, N_3 \}$, we have $\norm{L^{-1} - D}_{2} \leq \eps$, $d_i \leq 2/pt^2$, and $1/d_i \leq N^4$ for all $i$ with overall probability at least $1-\delta$.
\end{proof}

We used the following technical result. 
\begin{lemma}\label{lem:inversion}
Let $\norm{\cdot}$ be a matrix norm such that $\norm{AB} \leq \norm{A} \norm{B}$. Let matrices $C, E \in \R^{n\times n}$ be such that $\norm{C^{-1} E}_2 \leq 1$, and
let $\tilde{C} = C + E$. Then
\begin{equation}
\frac{\norm{\tilde{C}^{-1} - C^{-1}}} {\norm{C^{-1}}} \leq \frac{\norm{C^{-1} E}}{1 - \norm{C^{-1} E}}.
\end{equation}
\end{lemma}

This implies that if $\norm{E}_2 = \norm{\tilde{C} - C}_2 \leq 1/(2 \norm{C^{-1}}_2)$, then
\begin{equation}\label{eq:inverse-stability}
\norm{\tilde{C}^{-1} - C^{-1}}_2 \leq 2 \norm{C^{-1}}_{2}^{2} \norm{E}_{2}.
\end{equation}
\fi

In Theorem~\ref{thm:alg-1-correctness}, the diagonal entries are lower bounded, which avoids some degeneracy, but they could still grow quite large because of the heavy tails.
This is a real drawback of orthogonalization via covariance. HTICA, using the more sophisticated \emph{orthogonalization via centroid body scaling} does not have this problem. We can see this in the right table of Figure \ref{fig:stable-parameter-estimation-and-error}, where the condition number of ``centroid'' is much smaller than the condition number of ``covariance.''

\ifcolt
\else
\section{Membership oracle for the centroid body, without polarity}\label{sec:direct_membership}
We will see now how to implement an $\eps$-weak membership oracle for $\Gamma X$ directly, without using polarity. We start with an informal description of the algorithm and its correctness.
\fi

The algorithm implementing the oracle (Subroutine \ref{sub:centroid-oracle}) is the following: 
Let $q \in \RR^\dim$ be a query point. 
Let $X_1, \dotsc, X_N$ be a sample of random vector $X$.
Given the sample, let $Y$ be uniformly distributed in $\{X_1, \dotsc, X_N\}$. 
Output YES if $q \in \Gamma Y$, else output NO.

Idea of the correctness of the algorithm: If $q$ is not in $(\Gamma X)_\eps$, then there is a hyperplane separating $q$ from $(\Gamma X)_\eps$. 
Let $\{x \suchthat a^T x = b \}$ be the hyperplane, satisfying $\norm{a} = 1$, $a^T q > b$ and $a^T x \leq b$ for every $x \in (\Gamma X)_\eps$.
Thus, we have $h_{(\Gamma X)_\eps}(a) \leq b$ and $h_{\Gamma X}(a) \leq b - \eps$.
We have \[ h_{\Gamma Y}(a) = \e(\abs{a^T Y}) = (1/N) \sum_{i=1}^N \abs{a^T X_i}. \]
\iflong
By Lemma \ref{lem:1-plus-eps-chebyshev}, 
\fi
$(1/N) \sum_{i=1}^N \abs{a^T X_i}$ is within $\eps$ of $\e \abs{a^T X} = h_{\Gamma X}(a) \leq b - \eps$ when $N$ is large enough with probability at least $1-\delta$ over the sample $X_1, \dotsc, X_N$. 
In particular, $h_{\Gamma Y}(a) \leq b$, which implies $q \notin \Gamma Y$ and the algorithm outputs NO, with probability at least $1-\delta$. 

If $q$ is in $(\Gamma X)_{-\eps}$, let $y = q + \eps \hat q \in \Gamma X$.  
\iflong
We will prove the following claim: 
\fi

\iflong
Informal claim (Lemma \ref{lem:centroid_approximation}): 
\fi
For $p \in \Gamma X$, for large enough $N$ and with probability at least $1-\delta$ there is $z \in \Gamma Y$ so that $\norm{z - p} \leq \eps/10$. 

This claim applied to $p=y$ to get $z$, convexity of $\Gamma Y$ and the fact that $\Gamma Y$ contains $B \simeq \sigma_{\min}(A) B_2^\dim$ 
\iflong
(Lemma \ref{lem:innerball}) 
\fi
imply that $q \in \conv (B \cup \{ z \}) \subseteq \Gamma Y$ and the algorithm  outputs YES.

\iflong
We will prove the claim now. Let $p \in \Gamma X$. 
By the dual characterization of the centroid body (Proposition \ref{prop:dualcharacterization}), 
there exists a function $\lambda:\RR^\dim \to \RR$ such that $p = \e (\lambda(X) X)$ with $-1 \leq \lambda \leq 1$.
Let 
\(
z = \frac{1}{N} \sum_{i=1}^N \lambda(X_i) X_i.
\)
We have $\e_{X_i} (\lambda(X_i) X_i) = p$ and $\e_{X_i} ( \abs{\lambda(X_i) X_i}^{1+\gamma} ) \leq \e_{X_i} ( \abs{X_i}^{1+\gamma} ) \leq M$.
\iflong
By Lemma \ref{lem:1-plus-eps-chebyshev}
\else
By \cite[Lemma 14]{anon_htica}
\fi
and a union bound over every coordinate we get $\pr( \norm{p - z} \geq \eps) \leq \delta$ for $N$ large enough. 
\fi 

\iflong
\subsection{Formal Argument}
\fi
\iflong
\begin{lemma}[\cite{anon_htica}]\label{lem:centroid-scaling}
Let $S = (S_1, \dots, S_n) \in \RR^n$ be an absolutely symmetrically distributed random vector such that
$\e (\abs{S_i}) = 1$ for all $i$.
Then $B_{1}^{\dim} \subseteq \Gamma S \subseteq [-1,1]^\dim$.
Moreover, $\dim^{-1/2} B_2^\dim \subseteq (\Gamma S)^\polar \subseteq \sqrt{\dim}B_2^\dim$.
\end{lemma}

\begin{lemma}[\cite{anon_htica}]\label{lem:equivariance}
Let $X$ be a random vector on $\RR^\dim$.
Let $A: \RR^\dim \to \RR^\dim$ be an invertible linear transformation.
Then $\Gamma (AX) = A (\Gamma X)$.
\end{lemma}
\fi
\iflong
\begin{lemma}\label{lem:innerball}
Let $S = (S_1, \dots, S_n) \in \RR^n$ be an absolutely symmetrically distributed random vector such that
$\e (\abs{S_i}) = 1$ and $\e (\abs{S_i}^{1+\gamma}) \leq M < \infty$ for all $i$. 
Let $S^{(i)}, i=1, \dotsc, N$ be a sample of i.i.d. copies of $S$. Let $Y$ be a random vector, uniformly distributed  in $S^{(1)}, \dotsc, S^{(N)}$. 
Then $ (1-\eps') B_{1}^{\dim} \subseteq \Gamma Y$ whenever
\[
N \geq \left(\frac{16 M \dim^4}{(\eps')^2} \delta'\right)^{\frac{1}{2} + \frac{3}{\gamma}}.
\]
\end{lemma}
\fi
\begin{proof}
From Lemma \ref{lem:centroid-scaling} we know $\pm e_i \in \Gamma S$. 
It is enough to apply Lemma \ref{lem:centroid_approximation} to $\pm e_i$ with $\eps = \eps'/\sqrt{\dim}$ and $\delta = \delta'/(2\dim)$.
This gives, for any $\theta \in S^{\dim-1}$,  $h_{\Gamma Y} (\theta) \geq h_{\Gamma S} (\theta) -\eps \geq h_{B_1^\dim} (\theta) -\eps \geq (1-\sqrt{\dim} \eps) h_{B_1^\dim}(\theta) = (1-\eps')h_{B_1^\dim}(\theta)$. 
In particular, $\Gamma Y \supseteq (1-\eps') B_1^\dim$.
\end{proof} 

\begin{proposition}[Dual characterization of centroid body]\label{prop:dualcharacterization}
Let $X$ be a $\dim$-dimensional random vector with finite first moment, that is, for all $u \in \RR^\dim$ we have $\e (\abs{\inner{u}{X}}) < \infty$. Then 
\iflong
\begin{equation}\label{equ:dualcentroid}
    \Gamma X = \{ \e \bigl(\lambda(X) X \bigr) \suchthat \text{$\lambda:\RR^n \to [-1,1]$ is measurable}\}.
\end{equation}
\fi
\end{proposition}
\begin{proof}
Let $K$ denote the rhs of the conclusion.
We will show that $K$ is a non-empty, closed convex set and show that $h_K = h_{\Gamma X}$, which implies \eqref{equ:dualcentroid}.

By definition, $K$ is a non-empty bounded convex set. To see that it is closed, let $(y_k)_k$ be a sequence in $K$ such that $y_k \to y \in \RR^\dim$. 
Let $\lambda_k$ be the function associated to $y_k \in K$  according to the definition of $K$.
Let $\measure_X$ be the distribution of $X$.
We have $\norm{\lambda_k}_{L^\infty(\measure_X)} \leq 1$ and, passing to a subsequence $k_j$, $(\lambda_{k_j})$ converges to $\lambda \in L^\infty(\measure_X)$ in the weak-$*$ topology $\sigma(L^\infty(\measure_X), L^1(\measure_X))$, where $-1 \leq \lambda \leq 1$.
\footnote{This is a standard argument, see \cite{MR2759829} for the background. Map $x \mapsto x_i$ is in $L^1(\measure_X)$. \cite[Theorem 4.13]{MR2759829} gives that $L^1(\measure_X)$ is a separable Banach space. \cite[Theorem 3.16]{MR2759829} (Banach-Alaoglu-Bourbaki) gives that the unit ball in $L^\infty(\measure_X)$ is compact in the weak-* topology. \cite[Theorem 3.28]{MR2759829} gives that the unit ball in $L^\infty(\measure_X)$ is metrizable and therefore sequentially compact in the weak-* topology. Therefore, any bounded sequence in $L^\infty(\measure_X)$ has a convergent subsequence in the weak-* topology.}
This implies $\lim_{j} \e (\lambda_{k_j} (X) X_i) = \lim_j \int_{\RR^\dim} \lambda_{k_j} (x) x_i \ud \measure_X(x) = \int_{\RR^\dim} \lambda(x) x_i \ud \measure_X(x) = \e (\lambda (X) X_i) $.
Thus, we have $y = \lim_j y_{k_j} = \lim_j \e ((\lambda_{k_j}(X) X) = \e (\lambda(X) X)$ and $K$ is closed.

To conclude, we compute $h_K$ and see that it is the same as the definition of $h_{\Gamma X}$. In the following equations $\lambda$ ranges over functions such that $\lambda:\RR^n \to \RR$ is Borel-measurable and $-1 \leq \lambda \leq 1$.
\begin{align*}
h_K(\theta)
&= \sup_{y \in K} \inner{y}{\theta} \\
&= \sup_{\lambda} \e( \lambda(X) \inner{X}{\theta}) \\
\intertext{and setting $\lambda^*(x) = \sign \inner{x}{\theta}$,}
&= \e( \lambda^*(X) \inner{X}{\theta}) \\
&= \e (\abs{\inner{X}{\theta}}).
\end{align*}
\end{proof}

\begin{lemma}[LP]\label{lemma:centroid-lp}
Let $X$ be a random vector uniformly distributed in $\{x^{(i)}\}_{i=1}^{N} \subseteq \RR^\dim$. Let $q \in \RR^\dim$.
Then: 
\begin{enumerate}
\item $\Gamma X = \frac{1}{N} \sum_{i=1}^N [-x^{(i)}, x^{(i)}]$.

\item 
Point $q~\in~\Gamma X$ iff there is a solution $\lambda~\in~\RR^N$ to the following linear feasibility problem:
\begin{equation}\label{eq:lp}
\iflong
\begin{aligned}
&  \frac{1}{N} \sum_{i=1}^N \lambda_i x^{(i)} = q \\
& -1 \leq \lambda_i \leq 1 \quad \forall i.
\end{aligned}
\fi
\end{equation}
\item 
Let $\lambda^*$ be the optimal value of (always feasible) linear program
\begin{equation}\label{eq:lpminkowski}
\begin{aligned}
\iflong
&\lambda^* = \max \lambda \\
\text{s.t. } & \frac{1}{N} \sum_{i=1}^N \lambda_i x^{(i)} = \lambda q \\
& -1 \leq \lambda_i \leq 1 \quad \forall i
\fi
\end{aligned}
\end{equation}
with $\lambda^* = \infty$ if the linear program is unbounded. Then the Minkowski functional of $\Gamma X$ at $q$ is $1/\lambda^*$.
\end{enumerate}
\end{lemma}
%
\begin{proof}
\begin{enumerate}
\item
This is proven in \cite{MR0279689}.
It is also a special case of Proposition \ref{prop:dualcharacterization}. 
We include an argument here for completeness. 
Let $K := \frac{1}{N} \sum_{i=1}^N [-x^{(i)}, x^{(i)}]$. 
We compute $h_K$ to see it is the same as $h_{\Gamma X}$ in the definition of $\Gamma X$ (Definition \ref{def:centroid-body2}). 
As $K$ and $\Gamma X$ are non-empty compact convex sets, this implies $K = \Gamma X$.
We have 
\begin{align*}
h_K(y) 
&= \sup_{\lambda_i \in [-1,1]} \frac{1}{N} \sum_{i=1}^N \lambda_i x^{(i)} \cdot y \\
&= \max_{\lambda_i \in \{-1,1\}} \frac{1}{N} \sum_{i=1}^N \lambda_i x^{(i)} \cdot y \\
&= \frac{1}{N} \sum_{i=1}^N \abs{ x^{(i)} \cdot y }\\
&= \e (\abs{X \cdot y}).
\end{align*}
\item This follows immediately from part 1. 
\item This follows from part 1 and the definition of Minkowski functional.\qedhere
\end{enumerate}
\end{proof}
\begin{subroutine}[ht]
\caption{Weak Membership Oracle for $\Gamma X$}\label{sub:centroid-oracle}
\begin{algorithmic}[1]
\Require Query point $q \in \RR^\dim$,
samples from symmetric ICA model $X = AS$,
bounds $s_M \geq \sigma_{\max}(A)$, $s_m \leq \sigma_{\min}(A)$,
closeness parameter $\eps$,
failure probability $\delta$.
\Ensure $(\epsilon, \delta)$-weak membership decision for $q \in \Gamma X$.
\State Let $N = \poly(n, M, 1/s_m, s_M, 1/\eps, 1/\delta)$.
\State Let $(x^{(i)})_{i=1}^N$ be an i.i.d. sample of $X$. 
\State Check the feasibility of linear program \eqref{eq:lp}. If feasible, output YES, otherwise output NO.
\end{algorithmic}
\end{subroutine}

\begin{proposition}[Correctness of Subroutine \ref{sub:centroid-oracle}]
Let $X=AS$ be given by an ICA model such that for all $i$ we have $\e (\abs{S_i}^{1+\gamma}) \leq M < \infty$, $S_i$ is symmetrically distributed and normalized so that $\e \abs{S_i} = 1$.
Then, given a query point $q \in \RR^\dim$,  $\eps>0$, $\delta > 0$, $s_M \geq \sigma_{\max}(A)$, and $s_m \leq \sigma_{\min}(A)$, Subroutine~\ref{sub:centroid-oracle} is an $\eps$-weak membership oracle for $q$ and $\Gamma X$ with probability $1-\delta$ using time and sample complexity
\(
\poly(n, M, 1/s_m, s_M, 1/\eps, 1/\delta).
\)\lnote{query time should also depend on query, unless we are talking about arithmetic operations}
\end{proposition}
\begin{proof}
Let $Y$ be uniformly random in $(x^{(i)})_{i=1}^N$. There are two cases corresponding to the guarantees of the oracle: 
\begin{itemize}
\item Case $q \notin (\Gamma X)_\eps$. 
Then there is a hyperplane separating $q$ from $(\Gamma X)_\eps$.
Let $\{x \in \RR^\dim \suchthat a^T x = b \}$ be the separating hyperplane, parameterized so that $a \in \RR^\dim$, $b \in \RR$, $\norm{a} = 1$, $a^T q > b$ and $a^T x \leq b$ for every $x \in (\Gamma X)_\eps$. 
In this case $h_{(\Gamma X)_\eps}(a) \leq b$ and $h_{\Gamma X}(a) \leq b - \eps$.
At the same time, $h_{\Gamma Y}(a) = \e (\abs{a^T Y}) = (1/N) \sum_{i=1}^N \abs{a^T x^{(i)}}$.

We want to apply Lemma~\ref{lem:1-plus-eps-chebyshev} to $a^T X$ to get that $h_{\Gamma Y}(a) = (1/N) \sum_{i=1}^N \abs{a^T x^{(i)}}$ is within $\eps$ of $h_{\Gamma X}(a) = \e ( \abs{a^T X} )$. 
For this we need a bound on the $(1+\gamma)$-moment of $a^T X$. 
\ifcolt
\else
We use the bound from \cite[Equation (10)]{anon_htica}: 
\fi
$\e (\abs{a^T X}^{1+\gamma}) \leq M \dim s_M^{1+\gamma}$.
Lemma~\ref{lem:1-plus-eps-chebyshev} implies that for
\begin{align} \label{eqn:N-lower-bound}
N \geq 
\max \left\{ \left(\frac{8 M \dim s_M^{1+\gamma}}{\eps^2 \delta}\right)^{3/\gamma}, \left( \frac{8M \dim s_M^{1+\gamma}}{\eps}\right)^{\frac{1}{2} + \frac{1}{\gamma}} \right\},
\end{align}
we have
\[
P\left(\abs{\sum_{i=1}^N \abs{a^T x^{(i)}} - \e (\abs{a^T X})} > \eps\right) \leq \delta.
\]
In particular, with probability at least $1-\delta$ we have $h_{\Gamma Y}(a) \leq b$, which implies $q \notin \Gamma Y$ and, by Lemma \ref{lemma:centroid-lp}, Subroutine \ref{sub:centroid-oracle} outputs NO.

\item Case $q \in (\Gamma X)_{-\eps}$.
Let $y = q + \eps \hat q = q(1+\frac{\eps}{\norm{q}})$.
Let $\alpha = 1+\frac{\eps}{\norm{q}}$.
Then $y \in \Gamma X$.
Invoke Lemma \ref{lem:centroid_approximation} for i.i.d. sample $(x^{(i)})_{i=1}^N$ of $X$ with $p = y$ and $\eps$ equal to some $\eps_1 > 0$ to be fixed later to conclude $y \in (\Gamma Y)_{\eps_1}$.
That is, there exist $z \in \Gamma Y$ such that 
\begin{equation}\label{equ:z}
\norm{z - y} \leq \eps_1.
\end{equation}
Let $w = z/\alpha$.
Given \eqref{equ:z} and the relationships $y=\alpha q$ and $z=\alpha w$, we have 
\begin{equation}\label{equ:perturbation}
\norm{w-q} \leq \norm{z - y} \leq \eps_1.
\end{equation}
From Lemma \ref{lem:innerball} with $\eps'= 1/2$ and equivariance of the centroid body 
(Lemma \ref{lem:equivariance}) 
we get $\Gamma Y \supseteq \frac{\sigma_{\min}(A) }{2 \sqrt{\dim}} B_2^\dim$. This and convexity of $\Gamma Y$ imply $\conv\{\frac{\sigma_{\min}(A) }{2 \sqrt{\dim}} B_2^\dim \cup \{z\}\} \subseteq \Gamma Y$. In particular, the ball around $w$ of radius
\[
r := \left(1-\frac{1}{\alpha}\right) \frac{\sigma_{\min}(A) }{2 \sqrt{\dim}}
\]
is contained in $\Gamma Y$. 
The choice $\eps_1 = r \geq $ and \eqref{equ:perturbation} imply $q \in \Gamma Y$ and Subroutine \ref{sub:centroid-oracle} outputs YES whenever
\[
N \geq \left( \frac{8 M \dim^2}{r^2 \delta} \right)^{\frac{1}{2} + \frac{1}{\gamma}}.
\]
To conclude, remember that $q \in (\Gamma X)_{-\eps}$. Therefore $\norm{q} + \eps \leq \sqrt{\dim} \sigma_{\max}(A)$ (from Lemma \ref{lem:centroid-scaling} and equivariance of the centroid body, Lemma \ref{lem:equivariance}). This implies
\begin{align*}
r 
&= \frac{\eps}{\norm{q} + \eps} \frac{\sigma_{\min}(A) }{2 \sqrt{\dim}} \\
&\geq  \frac{\eps \sigma_{\min}(A) }{2 \dim \sigma_{\max}(A)}
\end{align*}
\end{itemize}
The claim follows.
\end{proof}
\iflong
\begin{lemma}\label{lem:centroid_approximation}
Let $X$ be a $\dim$-dimensional random vector such that for all coordinates $i$ we have $\e (\abs{X_i}^{1+\gamma}) \leq M < \infty$.
Let $p \in \Gamma X$. 
Let $(X^{(i)})_{i=1}^N$ be an i.i.d. sample of $X$.
Let $Y$ be uniformly random in $(X^{(i)})_{i=1}^N$. Let $\eps>0$, $\delta>0$. If $N \geq \left( \frac{8 M \dim^2}{\eps^2 \delta} \right)^{\frac{1}{2} + \frac{3}{\gamma}}$, then, with probability at least $1-\delta$, $p \in (\Gamma Y)_\eps$.
\end{lemma}
\fi
\begin{proof}
By Proposition \ref{prop:dualcharacterization}, there exists a measurable function $\lambda:\RR^\dim \to \RR$, $-1 \leq \lambda \leq 1$ such that $p = \e (X \lambda(X))$.
Let 
\[
z = \frac{1}{N} \sum_{i=1}^N X^{(i)} \lambda(X^{(i)}).
\]
By Proposition \ref{prop:dualcharacterization}, $z \in \Gamma Y$.

We have $\e_{X^{(i)}} (X^{(i)} \lambda (X^{(i)})) = p$ and, for every coordinate $j$,
\[
\e_{X^{(i)}} (\abs{X^{(i)}_j \lambda (X^{(i)})}^{1+\gamma}) 
\leq \e_{X^{(i)}}(\abs{X^{(i)}_j}^{1+\gamma}) 
\leq  M.
\]
By Lemma \ref{lem:1-plus-eps-chebyshev} and for any fixed coordinate $j$ we have, over the choice of $(X^{(i)})_{i=1}^N$,
\[
\pr( \abs{p_j - z_j} \geq \eps/\sqrt{\dim} ) 
\leq \frac{8 M}{(\eps/\sqrt{\dim})^2 N^{\gamma/3}} 
= \frac{8 M \dim}{\eps^2 N^{\gamma/3}}
\]
whenever $N \geq (8M\sqrt{n}/\eps)^{\frac{1}{2} + \frac{1}{\gamma}}$.
A union bound over $\dim$ choices of $j$ gives:
\[
\pr( \norm{p - z} \geq \eps) 
\leq \frac{8 M \dim^2}{\eps^2 N^{\gamma/3}}.
\]
So  
\(
\pr( \norm{p - z} \geq \eps) \leq \delta
\)
whenever
\[
N \geq \left( \frac{8 M \dim^2}{\eps^2 \delta} \right)^{3/\gamma}
\]
and $N \geq (8M\sqrt{n}/\eps)^{\frac{1}{2} + \frac{1}{\gamma}}$.
The claim follows.
\end{proof}




\fontsize{9.0pt}{10.0pt} \selectfont
\bibliography{ICA_bibliography}
\bibliographystyle{alpha}
\end{document}